\documentclass[numsec,webpdf,modern,medium,namedate]{oup-authoring-template}

\makeatletter

\def\ps@opening{%
  \def\@oddhead{}%
  \def\@evenhead{}%
  \def\@oddfoot{\hfil\thepage\hfil}%
  \def\@evenfoot{\hfil\thepage\hfil}}

\def\ps@headings{%
  \def\@oddhead{}%
  \def\@evenhead{}%
  \def\@oddfoot{\hfil\thepage\hfil}%
  \def\@evenfoot{\hfil\thepage\hfil}}

\makeatother

\onecolumn
\graphicspath{{Fig/}}
\usepackage[mathlines, switch]{lineno}

\theoremstyle{thmstyleone}%
\newtheorem{theorem}{Theorem}
\newtheorem{proposition}[theorem]{Proposition}%

\newtheorem{lemma}[theorem]{Lemma}%

\theoremstyle{thmstyletwo}%
\newtheorem{remark}{Remark}%

\theoremstyle{thmstylethree}%

\usepackage[onehalfspacing]{setspace}
\usepackage[fontsize=10pt]{fontsize}

\AtBeginDocument{
\raggedbottom

}
\usepackage{amsmath}
\usepackage{amssymb}
\usepackage{mathtools}
\usepackage{amsthm}
\usepackage{bbm}
\usepackage{amsfonts} 
\usepackage{xcolor}
\usepackage{nicefrac} 
\usepackage{amsbsy}
\usepackage{accents}
\usepackage{tikz}
\usetikzlibrary{arrows.meta} 
\usetikzlibrary{positioning}
\usetikzlibrary{bayesnet}
\usetikzlibrary{calc}

\usepackage{algpseudocode}

\begin{document}

\firstpage{1}

\title[Neural Generalized Mixed-Effects Models]{Neural Generalized Mixed-Effects Models}

\author[1]{Yuli Slavutsky}
\author[2]{Sebastian Salazar}
\author[1,2]{David M. Blei}

\address[1]{Department of Statistics, Columbia University, New York, USA}
\address[2]{Department of Computer Science, Columbia University, New York, USA}

\authormark{Slavutsky et al.}
\newcommand{\ubar}[1]{\underaccent{\bar}{#1}}

\newcommand\E{\mathbb{E}}
\newcommand\R{\mathbb{R}}
\newcommand\yvec{\boldsymbol{y}}
\newcommand\gammamat{\boldsymbol{\gamma}}
\newcommand\zmat{\boldsymbol{Z}}
\newcommand\xmat{\boldsymbol{X}}
\newcommand\gvec{\boldsymbol{g}}
\newcommand\profMat{\boldsymbol{W}}
\newcommand\supp{\mathrm{supp}}
\newcommand\D{\mathcal{D}}

\newcommand{\locloglik}[5]{\ell(#1,#2,#3;#4,#5)}
\newcommand{\locloglikM}[5]{\ell_{M}(#1,#2,#3;#4,#5)}

\newcommand{\locloglikMempt}{\log p_{M}}

\newcommand{\TruncL}[1]{\mathcal{L}_{#1}}
\newcommand{\EmpL}{\hat{\mathcal{L}}}
\newcommand{\FullL}{\mathcal{L}^*}

\newcommand{\TruncMaxSet}{\mathcal{A}_M}
\newcommand{\EmpMaxSet}{\hat{\mathcal{A}}}
\newcommand{\FullMaxSet}{\mathcal{A}^*}

\newcommand{\ParamSet}{\mathcal{A}}

\newcommand{\NumL}[1]{\tilde{\mathcal{L}}_{#1}}
\newcommand{\NumMaxSet}[1]{\tilde{\mathcal{A}}_{#1}}
\newcommand{\Hnum}{\tilde{H}}
\newcommand{\OdeErr}[1]{\varepsilon_{#1}}

\abstract{
Generalized linear mixed-effects models (GLMMs) are widely used to analyze grouped and hierarchical data. In a GLMM, each response is assumed to follow an exponential-family distribution where the natural parameter is given by a linear function of observed covariates and a latent group-specific random effect. Since exact marginalization over the random effects is typically intractable, model parameters are estimated by maximizing an approximate marginal likelihood. In this paper, we replace the linear function with neural networks. The result is a more flexible model, the \emph{neural generalized mixed-effects model} (NGMM), which captures complex relationships between covariates and responses. To fit  NGMM to data, we introduce an efficient optimization procedure that maximizes the approximate marginal likelihood and is differentiable with respect to network parameters. We show that the approximation error of our objective decays at a Gaussian-tail rate in a user-chosen parameter. On synthetic data, NGMM improves over GLMMs when covariate–response relationships are nonlinear, and on real-world datasets it outperforms prior methods. Finally, we analyze a large dataset of student proficiency to demonstrate how NGMM can be extended to more complex latent-variable models.
}

\maketitle
\thispagestyle{opening}
\pagestyle{headings}

\section{Introduction}

In many datasets, observations are organized into groups, such as students within schools, patients within clinics, or repeated measurements within subjects. In such settings, we can consider two sources of variation: a source that is shared across all groups, and a source that is group-specific. Mixed-effects models are designed for this structure. They decompose the effect of covariates into a \emph{fixed} component, which captures the relationship common to all groups, and a \emph{random} component, which captures group-specific deviations drawn from a common distribution.

The use of mixed-effects models dates back to the 1960s (see \citet{bryk1992hierarchical} for a historical perspective) and they are now widely employed to analyze grouped data. Specific applications include studies grouped by participants and experimental conditions \citep{baayen2008mixed, judd2012treating}, and by geographic regions \citep{gelman2007analysis}. Grouped designs are especially common in genome-wide association research \citep{zhou2012genome, listgarten2010correction, zhang2010mixed}, neuro-imaging \citep{bernal2013statistical, visscher2003mixed, payne2015revisiting}, and clinical trials \citep{van2023lecanemab}. Another canonical example is in education research (see, e.g., \citet{raudenbush1999synthesizing, nye2000effects, lyu2023estimating}), where students are nested within schools, which may themselves be nested within districts, cities, or countries.

Formally, a linear mixed-effects model (LMM) can be written as 
\begin{equation} \label{eq:linear}
    y_{ij} = x_{ij} \beta + z_{ij} \gamma_j + \epsilon_{ij},
\end{equation}
where $j\in \{1, \dots, m\}$ indexes the observed groups, $x_{ij} \in \R^p$ denotes the covariate vector for subject $i\in\{1,\dots, n_j\}$ in group $j$, $z_{ij} \in \R^{q}$ are group-level dependent covariates, and $y_{ij} \in \R$ is the corresponding response.
Our main case of interest is when $z_{ij}$ are identical for all $i$ in the same group, so that the random-effects covariates are group-specific but not subject-specific.

Here $\beta \in \R^p$ is an unknown \emph{fixed effect} coefficient vector, $\gamma_j$ is a per-group \emph{random effect} coefficient vector\footnote{Throughout the paper, expressions such as $x_{ij}\beta$ denote inner products and are written without transpose for notational simplicity.}, and $\epsilon_{ij}\sim \mathcal{N}(0, \sigma^2_\epsilon)$ is observational noise. A common special case is the \textit{random intercept model}, where $z_{ij}=1$ for all $i$ and $j$. When $z_{ij}$ varies across subjects within a group, the term $z_{ij}\gamma_j$ allows the effect of these covariates to differ by group. In that case, the model includes \emph{random slopes}; more generally, it allows group-specific deviations in the coefficients multiplying $z_{ij}$.

In a mixed-effects model, the random effects $\gamma_j \in \R^q$ are drawn from a prior distribution $\gamma_j \sim \pi$. It is usually assumed to have zero mean, which lets us interpret $\gamma_j$ as the deviation of group $j$ from the overall population mean. Most often, the random-effects prior is taken to be an independent, centered Gaussian distribution
$\gamma_j \sim \mathcal{N}(0, \Sigma)$, with $\Sigma = \sigma^2_\gamma I$.

Given a grouped dataset of covariate-response pairs, mixed-effects models are estimated by maximizing the marginal log likelihood. When the response is Gaussian and the random effects are Gaussian, the marginal likelihood has a closed form, which makes estimation substantially simpler than in the general case.

A natural extension of LMMs is \emph{generalized linear mixed models} (GLMMs) \citep{McCullaghNelder1989GLM}, which allow the response to follow any exponential-family distribution.
In a GLMM, the conditional mean of $y_{ij}$ is linked to the linear predictor through a known \emph{link function} $\eta$. Specifically, 
\begin{equation} \label{eq:glm}
\E[y_{ij} \mid  x_{ij}, z_{ij}, \beta, \gamma_j] = \mu_{ij},
\qquad
\mu_{ij} = \eta^{-1} \left( x_{ij}\beta + z_{ij}\gamma_j \right).
\end{equation}
For general exponential families, the marginal likelihood does not admit a closed form. Therefore, fitting and prediction require an approximation (see \S \ref{sec:related} for further discussion).

While GLMMs are widely used, many modern applications exhibit a nonlinear relationship between covariates and response.  To this end, we extend GLMMs by replacing their linear predictors with neural networks. Our \emph{neural generalized mixed-effects model} (NGMM) is given by
\begin{equation} \label{eq:our_model}
\E[y_{ij} \mid x_{ij}, z_{ij}, \gamma_j; \theta, \psi] = \mu_{ij},
\qquad
\mu_{ij} = \eta^{-1}\left(f_\theta(x_{ij}) + g_\psi(z_{ij}) \gamma_j \right),
\end{equation}
\vspace{1em}

Here $f_\theta$ and $g_\psi$ are neural networks with unknown parameters $\theta$ and $\psi$, and  $\gamma_j \in \R^k$ are group-specific latent variables, where
\begin{equation} 
\label{eq:prior}
    \gamma_j \sim \mathcal{N}(0, I_k).
\end{equation}
Since it uses neural networks, the NGMM can capture complex nonlinear relationships between covariates and response.
As in the linear case, our main focus is on the case where $z_{ij}$ are identical for all $i$ in the $j$-th group.

Our goal is to fit NGMM by maximizing the log marginal probability of the data. However, as with the linear GLMM, this objective is intractable and we must approximate it. Moreover, because we are fitting neural networks, our approximation should be compatible with automatic differentiation and optimization.

To address this, we use a change of variables to approximate the marginal likelihood with an ordinary differential equation (ODE). The resulting objective function is differentiable, and we maximize it with stochastic optimization, which scales to large datasets.

Once fitted, NGMM supports prediction for observations from both previously observed and new groups. For Gaussian and binary responses, it retains a key generalization property of classical mixed-effects models: for new groups, the Bayes-optimal predictor depends solely on the fixed-effects component.

This allows NGMM to outperform existing neural mixed-effects approaches that rely on closed-form marginal likelihoods in conjugate Gaussian or Poisson settings, while also accommodating non-conjugate cases such as Bernoulli outcomes. We show this on synthetic data, on prediction of Airbnb listing prices and bedroom counts, and on control-versus-experimental batch identification in microscopy images. Finally, we show that our approach generalizes to latent models 
by extending the multi-country PISA student proficiency study to a mixed-effects formulation with latent variable prediction.

The rest of this paper is organized as follows. Related work is discussed in \S \ref{sec:related}. 
In \S \ref{sec:method} we present the proposed estimation framework: the approximate objective is derived in \S \ref{sec:obj}, estimation and optimization are detailed in \S \ref{sec:estimation}, and prediction, including generalization to new groups, is discussed in \S \ref{sec:prediction}.
Theoretical properties are established in \S \ref{sec:theory}.
Synthetic experiments are reported in \S \ref{sec:synthetic}. 
Results on real data are reported in \S \ref{sec:experiments}.
In \S \ref{sec:airbnb} we compare NGMM with prior neural mixed-effects methods on the Airbnb dataset, and in \S \ref{sec:rxrx} we evaluate NGMM on binary prediction for microscopy images, a setting not supported by earlier approaches.
An analysis of an extended latent student proficiency model is presented in \S \ref{sec:pisa}. We conclude with a discussion in \S \ref{sec:discuss}.

\subsection{Related work} \label{sec:related}

In generalized linear mixed-effects models, the marginal likelihood no longer admits a closed form.
Therefore, fitting the model requires analytic or numerical approximations.
Common strategies include approximating the link or mean function (e.g., via probit or hyperbolic–tangent surrogates \citep{generalized}); applying Taylor expansions to the integrand; or using Laplace approximation. Perhaps most widely used is Gauss–Hermite quadrature, which numerically evaluates the integral of the product of the likelihood and the Gaussian prior over the random effects. For a detailed discussion and comparison of these methods, see Section 7.1.2 in \citep{demidenko2013mixed}.

However, once the linear component is replaced with a neural network, optimizing a marginal likelihood additionally requires the approximation to be differentiable with respect to the network weights. 
Due to this computational challenge, few approaches have previously incorporated neural networks into mixed effect models. 

\citet{Simchoni_Rosset1, Simchoni_Rosset2} addressed this challenge in the special case of an identity link, that is, a Gaussian response model, where the Gaussian prior on $\gamma_j$ and the Gaussian likelihood remain conjugate, and therefore the resulting marginal likelihood is differentiable with respect to the network parameters, enabling direct optimization. Similar methods for this conjugate Gaussian setting, differing mainly in their optimization strategies, were later proposed by \citet{kilian2023mixed, pmlr-v202-lee23k}.

Similarly, \citet{lee2023subject} assume a multiplicative Gamma random effect on the Poisson mean, which induces conjugacy and yields a closed-form marginal likelihood. In contrast, NGMM employs an additive Gaussian random effect, leading to a non-conjugate model.

Several methods have advanced the treatment of non-conjugate models by introducing surrogates for the exact marginal likelihood, often leading to complex optimization procedures.
In \citet{MeNets}, optimization requires repeated inversion of large covariance matrices to update the random effects, whereas the method of \citet{DeepGLMM} relies on a multi-stage procedure involving importance sampling, structured factor covariances, and variable selection. More recently, 
\citet{mandel2023neural} proposed a method that relies on Laplace-approximated penalized quasi-likelihood, while \citet{tschalzev2024enabling} use Monte Carlo sampling.

\section{NGMM: Neural Generalized Mixed-Effects Models} \label{sec:method}

In this section, we present the NGMM model, the parameter estimation procedure, and its use for prediction.

The NGMM model in Equations \eqref{eq:our_model} and \eqref{eq:prior} replaces the linear predictors of a GLMM with neural networks  $f_\theta: \R^p \to \R$ and $g_\psi: \R^q \to \R^k$, respectively. Here the parameters of the networks, $\theta$ and $\psi$, are treated as fixed effects. 
This yields a more flexible mean structure for the response, allowing a nonlinear dependence on the observation-level covariates $x_{ij}$.

The model is illustrated in Figure \ref{fig:illustration}.

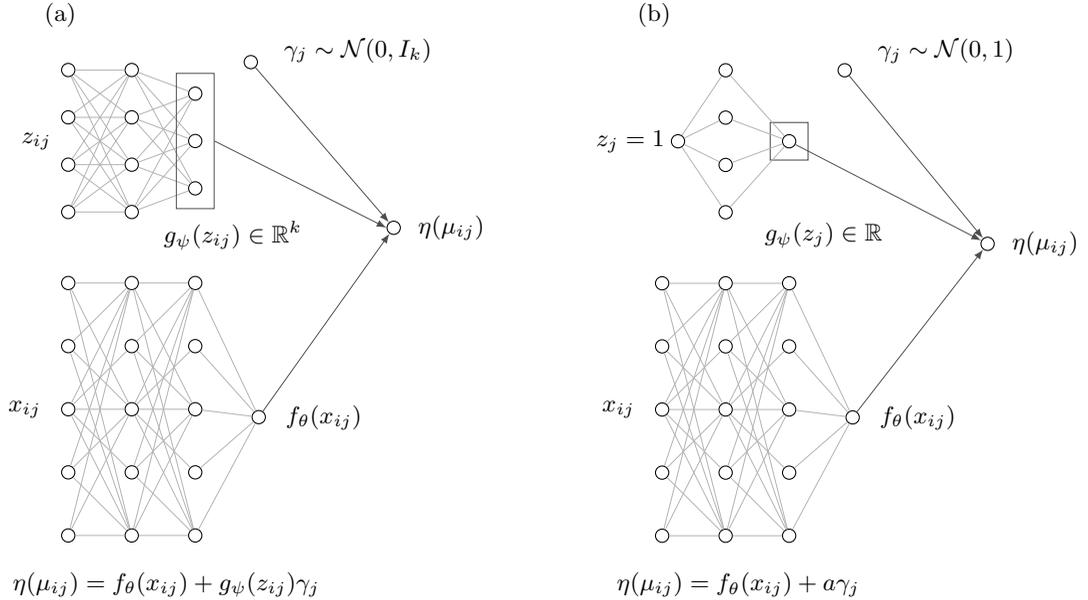
\begin{figure}
\centering

%======================= LEFT: general model (a) =======================
\begin{minipage}[c]{0.48\textwidth}
\centering
\begin{tikzpicture}[
  x=1.05cm,
  y=1.05cm,
  >=latex,
  neuron/.style={circle, draw, minimum size=5pt, inner sep=0pt},
  annot/.style={font=\small},
  conn/.style={line width=0.3pt, draw=black!30},
  arrowconn/.style={line width=0.3pt, draw=black!70, -latex},
  plate/.style={draw=black!70, line width=0.3pt, inner sep=0pt}
]

\node[annot] at (-2.5,2.7) {(a)};

\foreach \i in {1,...,4} {
  \node[neuron] (z\i)  at (-2.4, 0.6*\i-0.4) {};
  \node[neuron] (zh\i) at (-1.6, 0.6*\i-0.4) {};
}
\foreach \i in {1,...,4} {
  \draw[conn] (z\i) -- (zh1);
  \draw[conn] (z\i) -- (zh2);
  \draw[conn] (z\i) -- (zh3);
  \draw[conn] (z\i) -- (zh4);
}

\foreach \i in {1,...,3} {
  \node[neuron] (gz\i) at (-0.8, 0.6*\i-0.1) {};
}
\foreach \i in {1,...,4} {
  \draw[conn] (zh\i) -- (gz1);
  \draw[conn] (zh\i) -- (gz2);
  \draw[conn] (zh\i) -- (gz3);
}

\node[plate, minimum width=0.5cm, minimum height=1.8cm] (gplate) at (-0.8,1.1) {};

\node[annot] at (-2.8,1.1) {$z_{ij}$};
\node[annot,below=1pt of gplate] {\hspace*{1cm}$g_\psi(z_{ij}) \in \R^k$};

\node[neuron] (gamma) at (-0.1,2.1) {};
\node[annot,anchor=west] at (0.2,2.25) {$\gamma_j \sim \mathcal N(0,I_k)$};

\foreach \i in {1,...,5} {
  \node[neuron] (x\i)    at (-2.4,-0.8*\i+0.1) {};
  \node[neuron] (xh\i)   at (-1.6,-0.8*\i+0.1) {};
  \node[neuron] (xhII\i) at (-0.8,-0.8*\i+0.1) {};
}

\foreach \i in {1,...,5} {
  \draw[conn] (x\i) -- (xh1);
  \draw[conn] (x\i) -- (xh3);
  \draw[conn] (x\i) -- (xh5);
}
\draw[conn] (x3) -- (xh2);
\draw[conn] (x3) -- (xh4);

\foreach \i in {1,...,5} {
  \draw[conn] (xh\i) -- (xhII1);
  \draw[conn] (xh\i) -- (xhII3);
  \draw[conn] (xh\i) -- (xhII5);
}

\draw[conn] (xh3) -- (xhII2);
\draw[conn] (xh3) -- (xhII4);

\node[neuron] (fx) at (0,-2.4) {};
\foreach \i in {1,...,5} \draw[conn] (xhII\i) -- (fx);

\node[annot,left=4pt of x3] {$x_{ij}$};
\node[annot,right=4pt of fx] {$f_\theta(x_{ij})$};

\node[neuron] (eta) at (1.7,0.0) {};
\node[annot,right=3pt of eta] {$\eta(\mu_{ij})$};

\draw[arrowconn] (fx) -- (eta);
\draw[arrowconn] (gplate.east) -- (eta.west);
\draw[arrowconn] (gamma) -- (eta);

\end{tikzpicture}

\vspace{0.6em}
\noindent\hspace*{-2.0cm}{\small $\eta(\mu_{ij}) = f_\theta(x_{ij}) + g_\psi(z_{ij})\gamma_j$}
\vspace{1.2em}
\end{minipage}
\hfill
%======================= RIGHT: random intercept (b) ====================
\begin{minipage}[c]{0.48\textwidth}
\centering
\begin{tikzpicture}[
  x=1.05cm,
  y=1.05cm,
  >=latex,
  neuron/.style={circle, draw, minimum size=5pt, inner sep=0pt},
  annot/.style={font=\small},
  conn/.style={line width=0.3pt, draw=black!30},
  arrowconn/.style={line width=0.3pt, draw=black!70, -latex},
  outbox/.style={draw=black!70, rectangle, minimum width=0.5cm, minimum height=0.5cm, inner sep=0pt}
]

\node[annot] at (-2.5,2.7) {(b)};

\node[neuron] (gammaB) at (-0.1,2.0) {};
\node[annot,anchor=west] at (0.2,2.25) {$\gamma_j \sim \mathcal N(0,1)$};

\node[annot] (zrightlabel) at (-2.8,1.1) {$z_j = 1$};

\node[neuron] (zInB) at (-2.2,1.1) {};
\foreach \i in {1,...,4} {
  \node[neuron] (zhB\i) at (-1.6, 0.6*\i-0.4) {};
}

\node[outbox] (gboxB) at (-0.8,1.1) {};
\node[neuron] (zOutB) at (-0.8,1.1) {};
\node[annot,below=20pt of gboxB] {\hspace*{0.8cm} \small{$g_\psi(z_j)\in \R$}};

\foreach \i in {1,...,4} {
  \draw[conn] (zInB) -- (zhB\i);
  \draw[conn] (zhB\i) -- (zOutB);
}

\foreach \i in {1,...,5} {
  \node[neuron] (x\i)    at (-2.4,-0.8*\i+0.1) {};
  \node[neuron] (xh\i)   at (-1.6,-0.8*\i+0.1) {};
  \node[neuron] (xhII\i) at (-0.8,-0.8*\i+0.1) {};
}
\foreach \i in {1,...,5} {
  \draw[conn] (x\i) -- (xh1);
  \draw[conn] (x\i) -- (xh3);
  \draw[conn] (x\i) -- (xh5);
}
\draw[conn] (x3) -- (xh2);
\draw[conn] (x3) -- (xh4);

\foreach \i in {1,...,5} {
  \draw[conn] (xh\i) -- (xhII1);
  \draw[conn] (xh\i) -- (xhII3);
  \draw[conn] (xh\i) -- (xhII5);
}
\draw[conn] (xh3) -- (xhII2);
\draw[conn] (xh3) -- (xhII4);

\node[neuron] (fxB) at (0,-2.4) {};
\foreach \i in {1,...,5} \draw[conn] (xhII\i) -- (fxB);

\node[annot,left=4pt of x3] {$x_{ij}$};
\node[annot,right=4pt of fxB] {$f_\theta(x_{ij})$};

\node[neuron] (etaB) at (1.7,-0.2) {};
\node[annot,right=3pt of etaB] {$\eta(\mu_{ij})$};

\draw[arrowconn] (fxB) -- (etaB);
\draw[arrowconn] (gammaB) -- (etaB);
\draw[arrowconn] (zOutB) -- (etaB);

\end{tikzpicture}

\vspace{0.6em}
\noindent\hspace*{-2.5cm}{\small $\eta(\mu_{ij}) = f_\theta(x_{ij}) + a\gamma_j$}
\vspace{1.2em}
\end{minipage}

\caption{Illustration of the NGMM model. (a) General model. (b) Random-intercept special case with
$k = 1$, $z_j=1$, and $g_\psi(z_j) = a \in \R$.}
\label{fig:illustration}
\end{figure}

\subsection{Approximate marginal likelihood} \label{sec:obj}

Assume we are given a dataset $\D = \{\{(y_{ij},x_{ij}, z_{ij})\}_{i=1}^{n_j} \}_{j=1}^m$.
Denote $n=\sum_{j=1}^m n_j$ and let $\yvec$ collect all responses in $\D$. Similarly, let $\xmat \in \R^{n \times p}$ and $\zmat \in \R^{n \times q}$ be the corresponding covariate matrices.
Accordingly, let $\xmat_j \in \R^{n_j \times p}$, $\zmat_j \in \R^{n_j \times q}$, and $\yvec_j \in \R^{n_j}$ be the coordinates of the $j$-th group.
Our goal is to fit NGMM by maximizing marginal likelihood:

\begin{equation}
\left(\hat{\theta},\hat{\psi}\right)=\arg\max_{\theta,\psi}p(\yvec\mid \xmat,\zmat;\theta,\psi)
\end{equation}
where 
\begin{align}
p(\yvec \mid \xmat,\zmat;\theta,\psi)
&=
\prod_{j=1}^{m} p(\yvec_{j} \mid \xmat_j, \zmat_j;\theta,\psi)
=\prod_{j=1}^{m}
\int_{\mathbb{R}^{k}}
\left(
\prod_{i=1}^{n_j}
p(y_{ij} \mid x_{ij},z_{ij},\gamma;\theta,\psi)
\right)
\phi_{k}(\gamma) d\gamma,
\end{align}
with $\phi_k$ denoting the density of the $k$-dimensional Standard Gaussian distribution.

For clarity of presentation, we focus on a single group $j\in\{1,\dots,m\}$ and drop the group index when no ambiguity arises, writing $(x,y,z,\gamma)$ instead of $(\xmat_j,\yvec_j,\zmat_j,\gamma_j)$, where $y\in\supp(h)^{n_j}$ and $x\in\R^{n_j\times p}$.

Within a given group, conditionally on the covariates $\{(x_{ij}, z_{ij} )\}_{i=1}^{n_j}$ and $\gamma_j$, the outcomes $\{y_{ij}\}_{i=1}^{n_j}$ are independent. Equivalently, under the simplified notation, conditionally on $(x,z,\gamma)$, the entries of $y$ are independent, and the group-level marginal likelihood is
\begin{equation}\label{eq:q_dim_integral}
p(y \mid x,z;\theta,\psi)
=
\int_{\R^k}
p(y \mid x,z,\gamma;\theta,\psi)\,\phi_k(\gamma)\, d\gamma,
\end{equation}
where
\begin{equation}\label{eq:group_conditional_factor}
p(y \mid x,z,\gamma;\theta,\psi)
=
\prod_{i=1}^{n_j} p(y_i \mid x_i,z_i,\gamma;\theta,\psi),
\end{equation}
and $x_i$ and $z_i$ denote the $i$-th row of $x$ and $z$.

Below we focus on the main case where the \emph{random-effects design depends only on the group}, and address the general case in Appendix \ref{sup:multi}. In this case, $z_i=\mathrm{z} \in \R^q$ for all $i$ and therefore,
\begin{equation}\label{eq:shorthand_pyxt_0}
p(y \mid x,z,\gamma;\theta,\psi)
=
\prod_{i=1}^{n_j} p(y_i \mid x_i, \mathrm{z}, \gamma;\theta,\psi).
\end{equation}

The integral in \eqref{eq:q_dim_integral} generally has no closed form.
While this difficulty also arises in classical GLMMs, in our setting $\theta$ and $\psi$ parameterize neural networks. Their estimation therefore requires not only numerically approximating the marginal likelihood, but doing so in a way that is compatible with gradient-based optimization over $(\theta,\psi)$.

To meet this requirement, we derive an equivalent ordinary differential equation (ODE) formulation that replaces numerical integration, and whose solution admits propagation of derivatives with respect to $(\theta,\psi)$.

Given the neural predictors, for the $i$-th observation, conditionally on $(x_i, \mathrm{z},\gamma)$, the response $y_i$ follows a one-parameter exponential family, which can be expressed as
\begin{equation}\label{eq:expfam_general}
p(y_i \mid \vartheta_i)
=
h(y_i)\exp\Bigl(T(y_i)\,\vartheta_i
-
A\bigl(\vartheta_i \bigr)\Bigr),
\end{equation}
where $h(y_i)$ is the base measure, $\vartheta_i = \vartheta(x_i,\mathrm{z},\gamma; \theta,\psi)$ is the natural parameter, and $T$ is a sufficient statistic for $y_i$.
We focus on the standard GLM class of one-parameter exponential families, for which $T(y_i)=y_i$.
In this case the log-partition function is
\begin{equation}
A(t) \coloneq \log\int h(u)\exp\bigl(T(u)\,t\bigr) \, du,
\end{equation}
where $u$ denotes the integration variable over $\mathcal{Y}$ (and the integral is interpreted as a sum when $\mathcal{Y}$ is discrete).
The conditional mean satisfies
\begin{equation}\label{eq:mu1}
\E[y_i\mid \vartheta_i] = A'(\vartheta_i).
\end{equation}
Considering the link specification in Equation \eqref{eq:our_model},
\begin{equation}\label{eq:mu2}
\mu(x_i,\mathrm{z},\gamma; \theta,\psi)  =  \eta^{-1} \left(f_\theta(x_i) +g_\psi(\mathrm{z})\gamma \right),
\end{equation}
we require
\begin{equation}
A'(\vartheta_i) = \eta^{-1} \left(f_\theta(x_i) +g_\psi(\mathrm{z})\gamma \right).
\end{equation}
This determines the natural parameter as
\begin{equation}
\vartheta(x_i,\mathrm{z},\gamma;\theta,\psi)
=
(A')^{-1}\Bigl(\eta^{-1}\bigl(f_{\theta}(x_i)+g_{\psi}(\mathrm{z})\gamma\bigr)\Bigr),
\end{equation}
for an arbitrary (possibly non-canonical) link $\eta$.
In the canonical-link case, where $\eta(\mu)=(A')^{-1}(\mu)$, we obtain the simpler form
\begin{equation}
\vartheta(x_i,\mathrm{z},\gamma;\theta,\psi)=f_{\theta}(x_i)+g_{\psi}(\mathrm{z})\gamma.
\end{equation}

To obtain a differentiable approximation, we reduce the
integral in \eqref{eq:q_dim_integral} to a one-dimensional integral.
Denoting
\begin{equation}
t \coloneqq g_\psi(\mathrm{z})\gamma \in \R,
\end{equation}
we redefine the natural parameter in terms of $t$ as
\begin{equation}
\vartheta(x_i,t;\theta,\psi)
\coloneqq
(A')^{-1}\Bigl(\eta^{-1}\bigl(f_\theta(x_i)+t\bigr)\Bigr),
\end{equation}
and define the shorthand conditional density
\begin{equation}\label{eq:shorthand_pyxt}
p(y_i\mid x_i,t;\theta,\psi)
\coloneqq
h(y_i)\exp\left(y_i \, \vartheta(x_i,t;\theta,\psi)-A\bigl(\vartheta(x_i,t;\theta,\psi)\bigr)\right).
\end{equation}
Accordingly, Equation \eqref{eq:shorthand_pyxt_0} becomes
\begin{equation}\label{eq:group_conditional_t}
p(y \mid x,t;\theta,\psi)
=
\prod_{i=1}^{n_j} p(y_i\mid x_i,t;\theta,\psi).
\end{equation}

\begin{lemma}\label{lemma:int}
Fix a group $(x,y,\mathrm{z})$ and parameters $(\theta,\psi)$, and let $a=g_\psi(\mathrm{z})\in\mathbb{R}^k$.
Then
\begin{equation}
\int_{\mathbb{R}^{k}} p\left(y \mid x, a\gamma;\theta,\psi \right)\,\phi_k(\gamma)\,d\gamma
=
\int_{\mathbb{R}} p\left(y \mid x, \Vert a\Vert\xi;\theta,\psi \right)\,\phi_1(\xi)\,d\xi.
\end{equation}
\end{lemma}

This follows from the orthogonal invariance of the Gaussian distribution. For completeness, the proof is provided in Appendix \ref{sec:proof_lemma_int}.

Therefore, defining the scalar-parametrized natural parameter (now with $\xi$),
\begin{equation}\label{eq:vartheta_xz_xi}
\vartheta(x_i,\mathrm{z},\xi; \theta, \psi)
\coloneqq
(A')^{-1}\Bigl(\eta^{-1}\bigl(f_\theta(x_i) + \Vert g_\psi(\mathrm{z})\Vert\xi\bigr)\Bigr),
\end{equation}
and substituting this into the exponential-family form in \eqref{eq:expfam_general} with $T(y_i)=y_i$, we have
\begin{equation}
p(y \mid x,z;\theta,\psi)
=
\left(\prod_{i=1}^{n_j} h(y_i)\right)
\int_{\R}
\exp\left(
\sum_{i=1}^{n_j}
\left[
y_i \vartheta(x_i,\mathrm{z},\xi; \theta, \psi)
-
A(\vartheta(x_i,\mathrm{z},\xi; \theta, \psi))
\right]
\right)
\phi_1(\xi) d\xi .
\end{equation}
Denoting the log-integrand by
\begin{equation}\label{eq:L_def}
L(\xi; y,x,z, \theta, \psi)
\coloneqq
\sum_{i=1}^{n_j}
\left[
y_i \vartheta(x_i,\mathrm{z},\xi; \theta, \psi)
-
A(\vartheta(x_i,\mathrm{z},\xi; \theta, \psi))
\right]
+ \log \phi_1(\xi),
\end{equation}
we get
\begin{equation}\label{eq:log_exact}
\log p(y\mid x,z;\theta,\psi) =
\sum_{i=1}^{n_j}\log h(y_i) +
\log\int_{\R} \exp\left(L(\xi; y,x,z, \theta, \psi)\right) d\xi ,
\end{equation}
which we can approximate with the truncated version
\begin{equation}\label{eq:trunc}
\log p(y\mid x,z;\theta,\psi)
\approx
\sum_{i=1}^{n_j}\log h(y_i) +
\log \int_{\xi_0}^{\xi_1} \exp\left(L(\xi; y,x,z, \theta, \psi)\right) d\xi
\end{equation}
with appropriate values for $\xi_0$ and $\xi_1$.

Finally, define
\begin{equation}
F(\xi; y,x,z,\theta,\psi)
=
\int_{\xi_0}^{\xi} \exp\left(L(t; y,x,z, \theta, \psi)\right) dt,
\end{equation}
and, for $\xi$ such that $F(\xi; y,x,z,\theta,\psi)>0$,
\begin{equation}
H(\xi; y,x,z,\theta,\psi)
=
\log F(\xi; y,x,z,\theta,\psi).
\end{equation}
Then
\begin{equation}
\frac{dF(\xi; y,x,z,\theta,\psi)}{d\xi}
=
\exp\left(L(\xi; y,x,z, \theta, \psi)\right),
\end{equation}
and wherever $F(\xi; y,x,z,\theta,\psi)>0$, the chain rule gives
\begin{equation} \label{eq:ode_H}
\frac{dH(\xi; y,x,z,\theta,\psi)}{d\xi}
=
\exp\left(L(\xi; y,x,z, \theta, \psi) - H(\xi; y,x,z,\theta,\psi)\right).
\end{equation}

For $F(\xi_0; y,x,z,\theta,\psi) = 0$, this corresponds to the initial condition $H(\xi_0; y,x,z,\theta,\psi) = -\infty$.
In practice, we approximate this by starting from a finite baseline constant $H_0$.

Evaluating at $\xi_1$ yields
\begin{equation} 
H(\xi_1; y,x,z,\theta,\psi)
=
\log \int_{\xi_0}^{\xi_1} \exp\left(L(\xi; y,x,z, \theta, \psi)\right) d\xi.
\end{equation}
Combining this with Equation \eqref{eq:trunc} yields the approximation
\begin{equation} \label{eq:obj}
\log p(y \mid x,z;\theta,\psi)
\approx
\sum_{i=1}^{n_j}\log h(y_i)
+
H(\xi_1; y,x,z,\theta, \psi).
\end{equation}

\begin{remark}
    In the random‐intercept model, the original domain is one-dimensional. Therefore, the representation in Equation \ref{eq:log_exact} follows immediately. We illustrate such one–dimensional ODE construction in the special case of a random-intercept logistic model in Appendix \ref{sup:logistic}.
\end{remark}

\subsection{Estimation} \label{sec:estimation}

From a dataset $\D$, we aim to estimate $(\theta,\psi)$ by maximizing the truncated empirical objective induced by the group-level approximation in Equation \eqref{eq:obj}, that is 
\begin{equation}\label{eq:dataset_obj_trunc} \bar{\ell}(\theta,\psi) \coloneqq \frac{1}{m}\sum_{j=1}^{m} \Bigl[ \sum_{i=1}^{n_j}\log h(y_{ij}) + H(\xi_1; \yvec_j,\xmat_j,\zmat_j,\theta,\psi) \Bigr]. \end{equation}  

In practice, $H(\xi_1; \yvec_j,\xmat_j,\zmat_j,\theta,\psi)$ is obtained by numerically solving the ODE in Equation \eqref{eq:ode_H}; we denote the resulting numerical approximation by $\Hnum(\xi_1; \yvec_j,\xmat_j,\zmat_j,\theta,\psi)$ and maximize the corresponding numerical objective 
\begin{equation}\label{eq:dataset_obj_num} \tilde{\ell}(\theta,\psi) \coloneqq 
\frac{1}{m}\sum_{j=1}^{m} \Bigl[ \sum_{i=1}^{n_j}\log h(y_{ij}) + \Hnum(\xi_1; \yvec_j,\xmat_j,\zmat_j,\theta,\psi) \Bigr]. 
\end{equation} 

We maximize Equation \eqref{eq:dataset_obj_num} by stochastic gradient ascent using mini-batches of groups.
Let $\alpha_t=(\theta_t,\psi_t)$ denote the parameters at iteration $t$. At each iteration, we sample a mini-batch of groups $J\subset\{1,\dots,m\}$ uniformly at random, and define the mini-batch objective

\begin{equation}\label{eq:minibatch_obj_groups} \tilde{\ell}^{(t)}(\alpha_t) \coloneqq 
\frac{1}{\vert J\vert} \sum_{j\in J} \Bigl[ \sum_{i=1}^{n_j}\log h(y_{ij}) + \Hnum(\xi_1; \yvec_j,\xmat_j,\zmat_j,\alpha_t) \Bigr]. 
\end{equation} 
We then compute the stochastic gradient \begin{equation}\label{eq:sgd_gradient_main} 
g_t \coloneqq \nabla_{\alpha}\tilde{\ell}^{(t)}(\alpha_t), \end{equation} 
and update \begin{equation}\label{eq:sgd_update_main} \alpha_{t+1} = \alpha_t + \zeta \, g_t, 
\end{equation} where $\zeta>0$ is a step-size. 

We compute gradients by differentiating through a numerical ODE solver that produces $\Hnum(\xi_1;\yvec_j,\xmat_j,\zmat_j,\alpha_t)$, enabling end-to-end training. 
The optimization procedure is summarized in Algorithm \ref{alg:truncated_ode_sgd}. 

Since groups are sampled uniformly and $\tilde{\ell}(\alpha)$ is an average over groups, $g_t$ is a stochastic gradient estimator of $\nabla_{\alpha}\tilde{\ell}(\alpha_t)$. 
Note that if, in addition, we subsample observations within each sampled group when forming the group log-integrand, the resulting update is no longer an unbiased stochastic gradient of Equation \eqref{eq:dataset_obj_num}, as the map $L \mapsto H(\xi_1)$ is nonlinear. Nevertheless, this yields a computationally efficient approximation. Therefore, in Section \ref{sec:theory} we analyze group-level sampling, whereas in our large-scale experiments (see Section \ref{sec:experiments}) we also subsample observations within groups for efficiency.

\begin{algorithm}[t] 
\caption{Stochastic training with truncated ODE} 
\label{alg:truncated_ode_sgd} 
\begin{algorithmic}[1] 
\Require Dataset $\D$, truncation boundaries $\xi_0, \xi_1$, step size $\zeta > 0$, group batch size $b$, number of iterations $T$ 
\State Initialize parameters $\alpha_0=(\theta_0,\psi_0)$ \For{$t=0,\dots,T-1$} 
\State Sample a mini-batch of groups $J_t\subset\{1,\dots,m\}$ uniformly at random with $\vert J_t\vert=b$ \ForAll{$j\in J_t$} 
\State Form the log-integrand $L(\xi)\gets L(\xi;\yvec_j,\xmat_j,\zmat_j,\alpha_t)$ as in Equation \eqref{eq:L_def} 
\State Solve the ODE in Equation \eqref{eq:ode_H}  to obtain $\Hnum(\xi_1;\yvec_j,\xmat_j,\zmat_j,\alpha_t)$ 
\EndFor 
\State Compute $\tilde{\ell}^{(t)}(\alpha_t)$ as in Equation \eqref{eq:minibatch_obj_groups} 
\State Compute $g_t=\nabla_{\alpha}\tilde{\ell}^{(t)}(\alpha_t)$ 
\State Update $\alpha_{t+1}=\alpha_t+\zeta g_t$ 
\EndFor 
\State \textbf{return} $\alpha_T \eqqcolon (\hat{\theta},\hat{\psi})$ \end{algorithmic} \end{algorithm}

\subsection{Prediction} \label{sec:prediction}

Once fitted, NGMM can be used to predict outcomes for observations from previously observed groups $j \in \{1,\dots,m\}$ as well as from new groups $j^* \notin \{1,\dots,m\}$. As in classical mixed-effects models, the predictive procedure differs between these two cases.

In what follows, we restrict attention to the setting in which the random-effects design depends only on the group, so that $z_{ij}$ are identical for all subjects $i$ in group $j$. 

\paragraph{\emph{Observed groups.}}

For a new observation $i^*$ in an observed group $j$, the Bayes predictive distribution for the new response $y_{i^*j}$ is
\begin{equation}\label{eq:bayes_pred}
p(y_{i^*j} \mid x_{i^*j}, z_j, \yvec_j, \xmat_j ; \theta,\psi)
= \int p(y_{i^*j} \mid x_{i^*j}, z_j, \gamma_j; \theta,\psi)\, p(\gamma_j \mid \yvec_j, \xmat_j, \zmat_j;\theta,\psi)\, d\gamma_j,
\end{equation}
where the posterior for the random effect in group $j$ satisfies
\begin{align} \label{eq:posterior_gamma_j}
p(\gamma_{j}\mid\yvec_{j},\xmat_j, \zmat_j;\theta,\psi)
& \propto \phi_{k}(\gamma_{j})\,p(\yvec_{j}\mid \xmat_j, \zmat_j,\gamma_{j};\theta,\psi) \\
& = \phi_{k}(\gamma_{j})\prod_{i=1}^{n_{j}}p(y_{ij}\mid x_{ij},z_{j},\gamma_{j};\theta,\psi).
\end{align}

Using Lemma \ref{lemma:int}, the integrals in Equations \eqref{eq:bayes_pred} and \eqref{eq:posterior_gamma_j} can be reduced to one-dimensional integrals in $\xi$.
They can be evaluated by the same ODE construction as in Equation \eqref{eq:log_exact}, with the corresponding log-integrand obtained by combining the group-level sum over $(y_{ij},x_{ij},z_j)$ with the additional term contributed by $(y_{i^*j},x_{i^*j},z_j)$ in the numerator of Equation \eqref{eq:bayes_pred}.

Therefore, for Gaussian responses, NGMM extends a key property of LMMs: since the
conditional mean is linear in $\gamma_j$, the Bayes predictive mean
reduces to evaluation at the posterior mean
$\E[\gamma_j \mid \yvec_j,\xmat_j,\zmat_j;\theta,\psi]$.

For discrete
outcomes, we compute the predictive probabilities by evaluating the
integral in \eqref{eq:bayes_pred} for all candidate values
$y_{i^*j}$ and normalize.

\paragraph{\emph{New groups.}}
For a new group $j^*$, the posterior for $\gamma_{j^*}$ coincides with its prior, and therefore $\E[\gamma_{j^*}\mid z_{j^*}]=0$.
In LMMs, this renders the best linear unbiased predictor (BLUP) for an observation $(i,j^*)$ from a new group to reduce to the fixed-effects term
\begin{equation}
\hat{y}_{ij^*} = x_{ij^*}\hat{\beta}.
\end{equation}
This property is known as the \emph{generalization property} of linear mixed effects.

Here, NGMM preserves the analogous property.
For Gaussian responses with identity link,
\begin{equation}
\E \left[y_{ij^*} \mid x_{ij^*}, z_{j^*}\right]
=
\int \bigl(f_\theta(x_{ij^*}) + g_\psi(z_{j^*})\gamma_{j^*} \bigr)\, \phi_k(\gamma_{j^*})\, d\gamma_{j^*}
=
f_\theta(x_{ij^*}).
\end{equation}
Binary outcomes are addressed in the following proposition (see Appendix \ref{sec:proof_binary_new_groups} for proof).

\begin{proposition}[Prediction to new groups] \label{prop:binary_new_groups}
Let $y_{ij}\in\{0,1\}$, and let $\eta$ be a link function such that the inverse link $\eta^{-1}$ is strictly increasing and satisfies $1-\eta^{-1}(u)=\eta^{-1}(-u)$   (e.g., Logit and Probit).
For the corresponding NGMM model, the Bayes-optimal classifier for a new group $j^*\notin\{1,\dots,m\}$,
\begin{equation}
\hat y_{ij^*}
\coloneqq
\arg\max_{c\in\{0,1\}}
\, p(y_{ij^*}=c \mid x_{ij^*}, z_{j^*})
\end{equation}
depends only on the fixed-effects term $f_\theta(x_{ij^*})$ and does not depend on $g_\psi(z_{j^*})$.
\end{proposition}

\paragraph{\emph{Plug-in estimator.}}
Once NGMM is fit by maximizing the empirical
objective in \eqref{eq:dataset_obj_num} , in all cases, we use a plug-in estimator, replacing $(\theta,\psi)$ by their estimates
$(\hat{\theta},\hat{\psi})$.

\section{Analysis of estimation error}\label{sec:theory}

Assume the mixed-effects model in Equations \eqref{eq:our_model} and \eqref{eq:prior} is correctly specified. That is, the data is generated from our model for some true functions $f^{*}$ and $g^*$, and the chosen network classes are sufficiently expressive: there exist parameters $\alpha^* \coloneqq (\theta^{*}, \psi^{*})$ such that $f_{\theta^{*}}=f^{*}$ and $g_{\psi^{*}}=g^{*}$. Additionally, assume that the random effects depend only on group-level covariates, and
the training data $\D = \{\{(y_{ij},x_{ij})\}_{i=1}^{n_j},  z_{j}\}_{j=1}^m$ is obtained by two-stage sampling: first, $m$ groups are drawn i.i.d from the population of groups; then, conditional on each sampled group $j$, $n_j$ observations are drawn i.i.d from that group.

Setting aside for now the numerical errors due to the ODE solver’s finite precision (discussed in Appendix \ref{sup:numerical}), the NGMM estimator may deviate from the targets $f^{*}$ and $g^{*}$ for two reasons:
(i) a \emph{statistical error}, since the training data include only a sample of $m$ groups and a sample of observations within each sampled group; and
(ii) a \emph{truncation error}, stemming from replacing the exact group-level marginal log-likelihood in \eqref{eq:log_exact} by the truncated form in \eqref{eq:trunc}.
Next, we analyze these aspects.

To accommodate parameter symmetries of neural networks, we work with the induced functions rather than with a particular parameter representative. Recall that $\alpha \coloneqq (\theta,\psi)$ and define an equivalence relation
\begin{equation} \label{eq:equivalence}
\alpha \sim \alpha' \quad \Longleftrightarrow \quad f_{\theta}(\mathrm{x}) = f_{\theta'}(\mathrm{x}) \ \text{and}\ g_{\psi}(\mathrm{z}) = g_{\psi'}(\mathrm{z}) \qquad \forall \mathrm{x}, \mathrm{z},
\end{equation}
where $\mathrm{x}$ and $\mathrm{z}$ are single observation-level covariate values.

\subsection{Statistical error}

The truncated integral in Equation \eqref{eq:trunc} is not necessarily taken over a symmetric interval. Nevertheless, whenever the improper integral over $\R$ exists, its value coincides with the symmetric limit (see, e.g., the discussion following Theorem~10.33 in Section~10.13 of \citet{apostol1958mathematical}). Therefore, we analyze symmetric truncations and define
\begin{align}
& I(y;x,z,\theta,\psi)\coloneqq\int_{\R}\exp\bigl(L(\xi;y,x,z,\theta,\psi)\bigr)\,d\xi,\\
& I_{M}(y;x,z,\theta,\psi)\coloneqq\int_{-M}^{M}\exp\bigl(L(\xi;y,x,z,\theta,\psi)\bigr)\,d\xi.
\end{align}
We can now rewrite the exact group-level marginal log-likelihood in Equation \eqref{eq:log_exact}, and its population mean, respectively, as
\begin{equation}
\label{eq:ell_exact_again}
\locloglik{y}{x}{z}{\theta}{\psi}
=
\sum_{i=1}^{n_j}\log h(y_i)
+
\log I(y;x,z,\theta,\psi),
\quad
\FullL(\theta,\psi)
\coloneqq
\E\bigl[\locloglik{y}{x}{z}{\theta}{\psi}\bigr],
\end{equation}
where the expectation is taken with respect to the joint law of the hierarchical model under the two-stage sampling scheme.
Denote the maximizers of the population mean log-likelihood by
\begin{equation}
\FullMaxSet
\coloneqq
\left\{ \alpha^*=(\theta^*,\psi^*):\alpha^*\in\arg\max\FullL(\theta,\psi)\right\}.
\end{equation}
Similarly, define the truncated log-likelihood, and its population mean as
\begin{align}
& \locloglikM{y}{x}{z}{\theta}{\psi}
\coloneqq
\sum_{i=1}^{n_j}\log h(y_i)
+
\log I_{M}(y;x,z,\theta,\psi),\\
&
\TruncL{M}(\theta,\psi)
\coloneqq
\E\bigl[\locloglikM{y}{x}{z}{\theta}{\psi}\bigr],
\label{eq:ell_M_again}
\end{align}
as well as the corresponding maximizers,
\begin{equation}
\TruncMaxSet
\coloneqq
\left\{ \alpha_M=(\theta_M,\psi_M):\alpha_M\in\arg\max\TruncL{M}(\theta,\psi)\right\}.
\end{equation}

Assume that, for each fixed $M$, the population maximizers of the truncated criterion are unique up to the equivalence relation in \eqref{eq:equivalence}, that is, the set of maximizers forms a single equivalence class. Define $f_M \coloneqq f_{\theta_M}$ and $g_M \coloneqq g_{\psi_M}$ for any $\alpha_M\in \TruncMaxSet$.

The NGMM estimator maximizes the empirical analogue
\begin{align}
& \bar{\ell}_{M,j}(\theta,\psi)
\coloneqq
\sum_{i=1}^{n_j}\log h(y_{ij})
+
\log I_M(\yvec_j;\xmat_j,\zmat_j,\theta,\psi),\\
&
\EmpL(\theta,\psi)
=
\frac{1}{m}\sum_{j=1}^{m}\bar{\ell}_{M,j}(\theta,\psi),
\end{align}
and thus targets
\begin{equation}
\EmpMaxSet
\coloneqq
\left\{ \hat{\alpha}=(\hat{\theta},\hat{\psi}):\hat{\alpha}\in\arg\max\EmpL(\theta,\psi)\right\}.
\end{equation}

Under the two-stage sampling scheme above, each group $j$ produces a random cluster $W_j=\{z_{j}\} \cup \{(y_{ij}, x_{ij})\}_{i=1}^{n_j}$, and the empirical criterion $\EmpL(\theta,\psi)$ is an M-estimator based on i.i.d clusters $W_1, \dots, W_m$.
Under standard regularity conditions (continuity, dominated envelope, and uniqueness of the maximizer up to the equivalence relation above) M-estimation theory for clustered data implies that the empirical maximizers $\EmpMaxSet$ converge in probability to $\TruncMaxSet$ as $m \to \infty$, and that any smooth finite-dimensional functional of $\hat{\alpha}\in\EmpMaxSet$ is asymptotically normal at rate $\sqrt{m}$.
See, for example, Section 5.2, Theorems 5.7 and 5.23 in \citet{van2000asymptotic} for general M-estimation, and Section 3.6.2 in \citet{demidenko2013mixed} for the mixed-effects case under the stochastic cluster scheme.

Therefore, statistical error around the truncated targets is handled by standard M-estimation theory.
This leaves the question of \emph{how the truncated targets $f_M \coloneqq f_{\theta_M}$ and $g_M \coloneqq g_{\psi_M}$ relate to the exact targets $(f^{*},g^{*})$.} We address this question in the following section.

\subsection{Truncation error}

Denoting
\begin{equation}
\Delta_M(y;x,z,\theta,\psi)
\coloneqq
\log I(y;x,z,\theta,\psi)-\log I_{M}(y;x,z,\theta,\psi),
\end{equation}
the pointwise deviation of $(f_M,g_M)$ from $(f^{*},g^{*})$ stems from the truncation error $\Delta_M(y;x,z,\theta,\psi)$.
The following theorem (see Appendix \ref{sec:bound_proof} for proof) quantifies this truncation error: the difference between the exact log-integral and its truncated form decays at a Gaussian tail rate as the truncation interval expands.

\begin{theorem}[Pointwise truncation error] \label{thm:bound}
Let $p(u \mid \vartheta)$ be a regular one–parameter exponential family with natural statistic $T(u)=u$.
Fix covariates $(x,z)$ and parameters $(\theta,\psi)$.
Then, for all $y \in \supp(h)^{n_j}$, there exist constants $M_0>0$ and $C>0$ such that for all $M\ge M_0$,
\begin{equation}
0 \le \Delta_M(y)
\leq
C\frac{1}{M}
\exp\Bigl(-\tfrac{1}{2} M^{2}\Bigr),
\end{equation}
and therefore, 
\begin{equation}
\limsup_{M\to\infty}\frac{1}{M^{2}}\log \Delta_M(y)\le -\frac{1}{2}.
\end{equation}
\end{theorem}

Intuitively, Theorem~\ref{thm:bound} controls the truncation error for fixed parameters and covariates: for any fixed $(x,z)$ and $(\theta,\psi)$, the difference between the exact and truncated log-integrals, $\Delta_M(y;x,z,\theta,\psi)$, decays at a Gaussian tail rate. 

%%%%%%%%%%%%%%%%%%%%%%%%%%%%

In two special cases, Gaussian and Probit, we can characterize the truncation error and its asymptotic decay rate explicitly.
We use these settings to empirically illustrate the truncation error and its asymptotic decay rate in Figure \ref{fig:truncation}.
Closed-form expressions for the Gaussian integrals and the corresponding truncation error are provided in Appendix~\ref{sec:explicitntegrals}.
For the Probit model, Appendix~\ref{sec:explicitntegrals} provides explicit integral expressions and derives the Gaussian tail exponent of the truncation error.

In the Probit model, the generic bound from Theorem~\ref{thm:bound} has the correct worst-case Gaussian tail exponent $1/2$ in $M^{2}$.

\begin{figure} [ht]
    \centering
    \includegraphics[width=\linewidth]{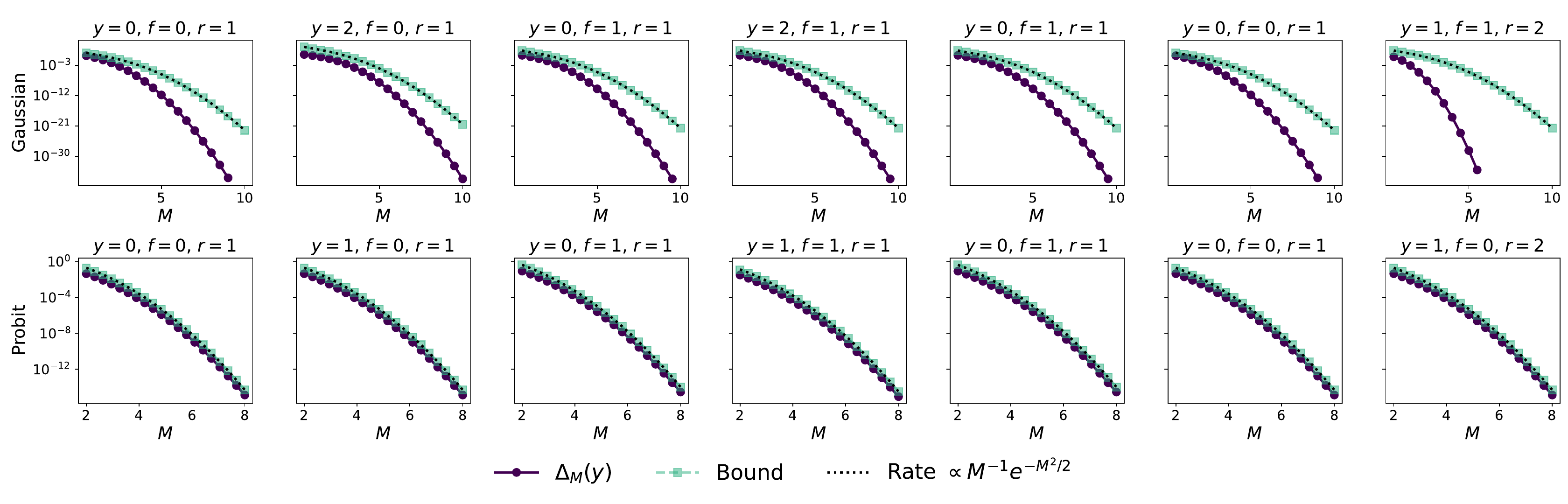}
    \caption{Truncation error and its bound. Gaussian-tail rate is scaled to match the bound at the largest $M$.
    Top (Gaussian): $\Delta_M(y)\asymp M^{-1}e^{-C_G M^2}$ with $C_G=(1+n_j r^2)/2$, while the bound has $M^{-1}e^{-M^2/2}$ rate.
    Bottom (Probit): $\Delta_M(y)$ has Gaussian tail exponent at least $1/2$ in $M^2$ (worst-case $1/2$), matching the generic bound in the worst-case sense.}
    \label{fig:truncation}
\end{figure}

Next, we use the pointwise bound in Theorem \ref{thm:bound} to derive a population-level bound, uniform over parameters $\alpha$, showing that $\TruncL{M}$ converges uniformly to $\FullL$.

Recall that for each fixed $\alpha$ and $(y,x,z)$, Theorem~\ref{thm:bound} yields constants
$M_0=M_0(y;x,z,\alpha) > 0$ and $C=C(y;x,z,\alpha) > 0$ such that
\begin{equation}
\label{eq:pointwise_bound_again_alpha}
0 \le \Delta_M(y;x,z,\alpha)
\le
\frac{C}{M}
\exp\Bigl(-\tfrac{1}{2} M^{2}\Bigr)
\quad\text{for all } M \ge M_0.
\end{equation}

Define population truncation constants for each $\alpha$ as
\begin{equation}
M_0^{\mathrm{p}}(\alpha)
\coloneqq
\sup_{(y,x,z)} M_0(y;x,z,\alpha),
\qquad
\overline C(\alpha)
\coloneqq
\sup_{(y,x,z)} C(y;x,z,\alpha),
\end{equation}
and assume that $M_0^{\mathrm{p}}(\alpha) < \infty$ and $\overline C(\alpha) < \infty$.
If $\ParamSet$ is a compact parameter set, then
\begin{equation}
M_0^{\ParamSet}
\coloneqq
\sup_{\alpha \in \ParamSet} M_0^{\mathrm{p}}(\alpha) < \infty,
\qquad
\overline C^{\ParamSet}
\coloneqq
\sup_{\alpha \in \ParamSet} \overline C(\alpha) < \infty.
\end{equation}
Under these assumptions, for every $\alpha \in \ParamSet$ and every $M \ge M_0^{\ParamSet}$, the bound in Equation \eqref{eq:pointwise_bound_again_alpha} implies
\begin{equation}
0 \le \Delta_M(y;x,z,\alpha)
\le
\frac{\overline C^{\ParamSet}}{M}
\exp\Bigl(-\tfrac{1}{2} M^{2}\Bigr)
\quad\text{almost surely}.
\end{equation}
Taking expectations over the joint law of the hierarchical model yields
\begin{align}
0
\le
\FullL(\alpha) - \TruncL{M}(\alpha)
& =
\E\bigl[\Delta_M(y;x,z,\alpha)\bigr]
\le
\frac{\overline C^{\ParamSet}}{M}
\exp\Bigl(-\tfrac{1}{2} M^{2}\Bigr)
\end{align}
for all $\alpha \in \ParamSet$.
In particular,
\begin{equation}
\sup_{\alpha \in \ParamSet}
\bigl\vert \FullL(\alpha) - \TruncL{M}(\alpha)\bigr\vert
\to 0
\quad \text{as } M \to \infty,
\end{equation}
so $\TruncL{M}$ converges uniformly to $\FullL$ on $\ParamSet$.
By the Argmax Theorem (see, e.g., \citealp[Chapter~5]{Vaart_1998}),
any sequence of maximizers
$\hat\alpha_{M} \in \arg\max_{\alpha \in \ParamSet} \TruncL{M}(\alpha)$ is
asymptotically maximizing for the full model,
\begin{equation}
\FullL(\hat\alpha_{M})
\to
\sup_{\alpha \in \ParamSet} \FullL(\alpha),
\end{equation}
and, whenever $\FullL$ admits a unique maximizer
$\alpha^{*} \in \ParamSet$, we have $\hat\alpha_{M} \to \alpha^{*}$.

\begin{remark}
In practice, for neural networks, these uniform bounds are approximately satisfied if the covariates $(x,z)$ take values in a bounded region, the parameter space is effectively restricted (for example by weight regularization), and the network maps $(f_{\theta},g_{\psi})$ are continuous.
In such settings, the constants $C(y;x,z,\theta,\psi)$ vary smoothly over $(y,x,z)$ and are uniformly bounded on the support of $(y,x,z)$, and the corresponding thresholds $M_0(y;x,z,\theta,\psi)$ can be chosen uniformly.
\end{remark}

\section{Simulations} \label{sec:synthetic}

We begin by evaluating NGMM on grouped synthetic data with a correctly specified nonlinear data-generating process.  
We set the number of groups to $m = 50$ and the number of observations per group to $n_j = 40$ for all $j \in \{1,\dots,m\}$. For each observation $(i,j)$, we independently draw  covariates $x^{(1)}_{ij}, x^{(2)}_{ij} \sim \mathcal{N}(0,1)$
and a group-specific random effect $\gamma_j \sim \mathcal{N}(0,\tau^2)$ with $\tau \in \{0.1, 0.3, 0.5\}$.
We define the fixed-effect and random-effect functions as
\begin{equation}
f^{*}\left(x^{(1)},x^{(2)}\right)
=
\sin\left(x^{(1)} + x^{(2)}\right)
+
x^{(1)} x^{(2)}, \qquad g^{*}\left(x^{(1)}\right) = x^{(1)}.
\end{equation}
We consider two random-effects structures:
a \emph{random-intercept} setting
\begin{equation}
\eta_{ij}
=
f^{*}\left(x^{(1)}_{ij}, x^{(2)}_{ij}\right)
+
\gamma_j,
\end{equation}
and a 
\emph{random-slope}
\begin{equation}
\eta_{ij}
=
f^{*}\left(x^{(1)}_{ij}, x^{(2)}_{ij}\right)
+
g^{*}\left(x^{(1)}_{ij}\right)\gamma_j.
\end{equation}

Given $\eta_{ij}$, we generate responses from three models: (i) Gaussian $y_{ij} = \eta_{ij} + \varepsilon_{ij}$ with $\varepsilon_{ij} \sim \mathcal{N}(0,1)$, (ii) Logistic $y_{ij} \sim \mathrm{Bernoulli}(p_{ij})$ with $\log\frac{p_{ij}}{1 - p_{ij}} = \eta_{ij}$, and (iii) Poisson $y_{ij} \sim \mathrm{Poisson}(\lambda_{ij})$ with $\log \lambda_{ij} = \eta_{ij}$.

For each combination of outcome family (Gaussian, Logistic, Poisson), random-effects structure (intercept or slope), and variance level $\tau\in\{0.1,0.3,0.5\}$, we evaluate on two test sets: (i) new observations from \emph{previously observed groups}, and (ii) observations from \emph{new groups}.

We compare our NGMM to standard GLMM baselines.\footnote{We use \texttt{MixedLM} from \texttt{statsmodels}. For logistic and Poisson outcomes we fit \texttt{glmer} models from the \texttt{lme4} R package with logit and log links, respectively.} For our model, we set $f$ and $g$ to be multilayer perceptrons (MLPs) with three hidden layers of sizes $16$, $8$, and $4$, and ReLU activations.

We evaluate Gaussian models using root mean-squared error (RMSE), logistic models using Bernoulli log-loss (negative marginal likelihood), and Poisson models using Poisson deviance. 
Across all settings, we report the relative change in each metric $V$ as $\frac{V_{\text{ours}}-V_{\text{GLMM}}}{V_{\text{GLMM}}}$. Additional details are provided in Appendix \ref{sup:synthetic}.

\paragraph{\emph{Results.}}
Results for the 33 resulting experiments over 10 repetitions are reported in Figure \ref{fig:simulations_nonlinear}. In the Poisson random-slope setting for previously observed groups, the GLMM baseline failed to converge and was therefore omitted from the analysis.
Since the true relationship is nonlinear, in all cases our model achieves lower errors.

\begin{figure} [ht]
    \centering
    \includegraphics[width=\linewidth]{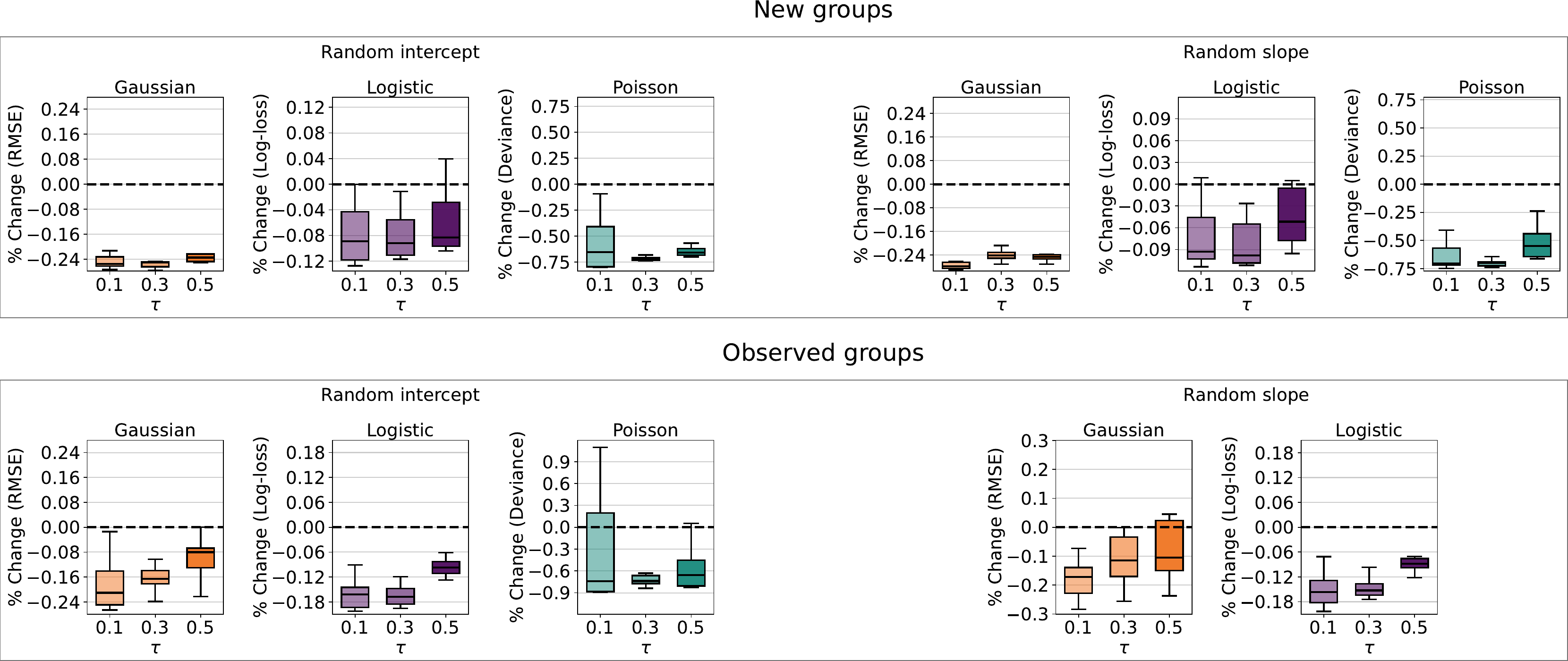}
    \caption{Simulation results. Values below zero indicate lower error for NGMM compared to GLMM.}
    \label{fig:simulations_nonlinear}
\end{figure}

Analogous results for linear data-generating processes are reported in Appendix \ref{sup:simualations}. In those experiments, we constrain the networks to be linear, thereby matching the functional form of the GLMM baselines. Under correct linear specification, GLMMs exhibit a slight advantage in some settings due to differences in the marginal-likelihood approximation. However, in 10 of the 33 cases NGMM outperforms GLMMs, and in the remaining cases the performance gaps are substantially smaller than the gains observed under nonlinear specifications.

\section{Experiments} \label{sec:experiments}
\subsection{Gaussian and Poisson neural mixed-effects models on Airbnb} \label{sec:airbnb}

In two special cases, Gaussian and Gamma-Poisson models, prior work exploited closed-form marginal likelihoods. To compare NGMM to these models,
we evaluate on an Airbnb listings dataset with suitable targets for both settings.

We treat host identity as the grouping variable and consider two prediction tasks: (i) listing price (treated with Gaussian likelihood) and (ii) bedroom count (treated with Poisson likelihood). In each repetition of the experiment, we subsample 500 hosts and evaluate all methods on (a) observations from hosts seen during training (using both fixed and random effects) and (b) observations from previously unseen hosts (using only fixed effects).

Across all settings, the fixed-effects component $f_\theta$ is an MLP with one hidden layer of size $32$, and the random intercept is a linear map from a host indicator to a scalar.

For the Gaussian task, we compare to LMMNN \citep{Simchoni_Rosset1, Simchoni_Rosset2}, which is specialized to Gaussian responses and optimizes the exact marginal likelihood. 
For the Poisson task, we compare to the Poisson–Gamma neural mixed-effects model (PGNN) of \citet{lee2023subject}.
Additionally, following \citet{Simchoni_Rosset1, Simchoni_Rosset2} we compare to two non-hierarchical baselines that encode group structure via one-hot group indicators (OHE) or learned group embeddings (Embed). Additional details are provided in Appendix \ref{sup:airbnb} and results over 10 repetitions are reported in Figure \ref{fig:airbnb}.

\begin{figure} [ht]
    \centering
    \includegraphics[width=\linewidth]{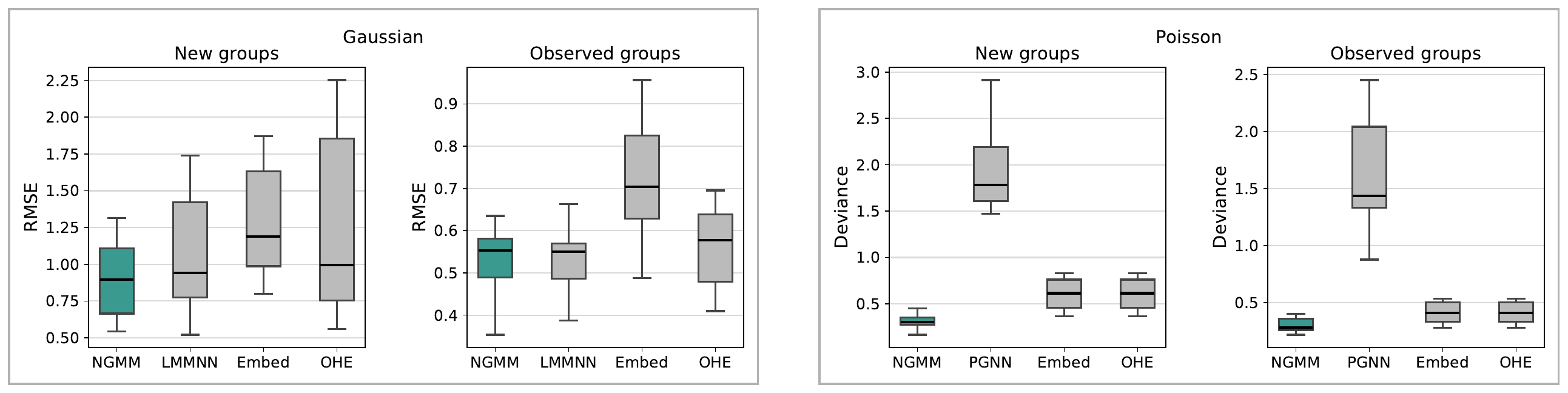}
    \caption{Airbnb experimental results. For the Gaussian response, NGMM (this paper) achieves comparable results to LMMNN \citep{Simchoni_Rosset1}. For the Poisson response, NGMM outperforms all other methods including PGNN \citep{lee2023subject}.}
    \label{fig:airbnb}
\end{figure}

\subsection{Logistic neural mixed-effects model for RxRx1 microscopy images} \label{sec:rxrx}

Prior neural mixed-effects models have focused on conjugate settings, treating Gaussian and Poisson responses, but not binary responses. To study NGMM with  binary responses, we analyze the RxRx1 dataset \citep{taylor2019rxrx1}, showing that modeling random effects remains beneficial for high-dimensional imaging data.  

The RxRx1 dataset consists of fluorescence microscopy images collected under many siRNA perturbations across multiple experimental batches. Our binary classification problem is to distinguish untreated control wells from wells treated with a single active siRNA. In each repetition, we choose one siRNA with sufficient support as the positive class and set the negative class to control. 

We treat imaging site as the grouping variable, capturing systematic variation across microscope fields-of-view. We consider two evaluation regimes. In the \emph{old groups} setting, we train on the RxRx1 training split and evaluate on an in-distribution split that contains only sites observed during training, using both fixed and random effects at test time. In the \emph{new groups} setting, we train on the same training split but evaluate on the WILDS out-of-distribution split, which contains only previously unseen sites; here predictions rely solely on the fixed-effects.

For all methods, we set the fixed-effects component $f_\theta$ to a small convolutional network with two $3\times 3$ convolutional layers (16 and 32 channels) followed by  average pooling and two fully connected layers of width $64$. NGMM augments this network with a Gaussian random intercept per site. 

As before, we compare against non-hierarchical baselines that encode the site as one-hot indicators (OHE) or learned site embeddings (Embed). Additional details are in Appendix \ref{sup:rxrx} and results over 10 repetitions are reported in Figure \ref{fig:rxrx}.

\begin{figure} [ht]
    \centering
    \includegraphics[width=\linewidth]{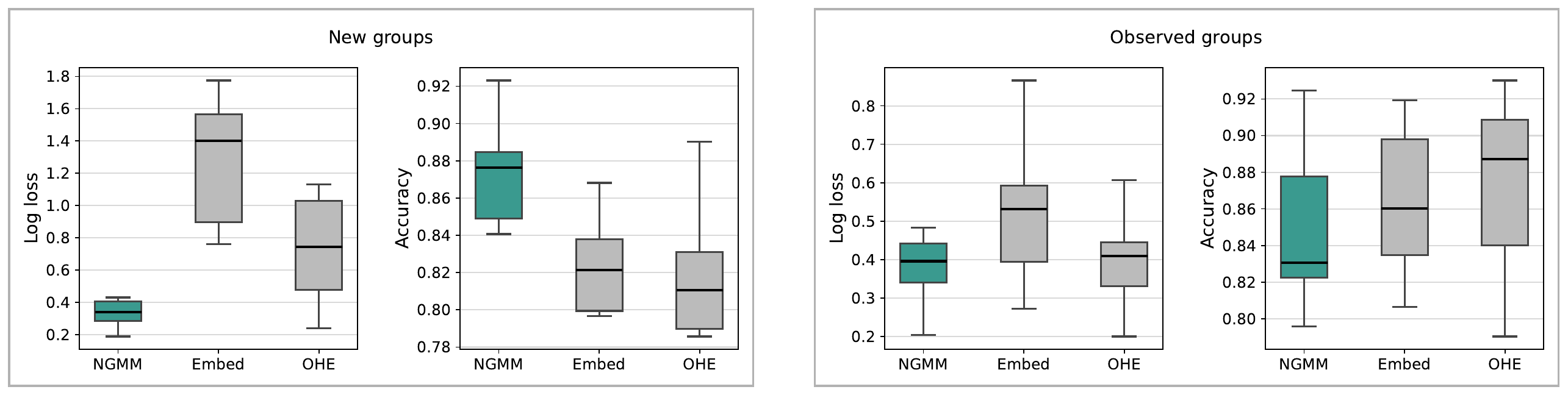}
    \caption{RxRx1 experimental results. Both for new groups and observed groups NGMM approximates the true marginal log-likelihood better than existing methods, achieving lower log-loss. For new groups it achieves higher prediction accuracy than all methods, and higher than BBVI for observed groups. On observed groups, the OHE and Embed baselines obtain higher accuracy but worse log-loss than NGMM, indicating overfitting of their predictive probabilities to the training sites.}
    \label{fig:rxrx}
\end{figure}

\subsection{Latent mixed-effects neural model for PISA student proficiencies} \label{sec:pisa}

Here we show that NGMM extends beyond observed-response settings to latent-variable models. To demonstrate this, we revisit the 2012 Programme for International Student Assessment (PISA) student proficiency analysis \citep{pisa}, where student proficiency is treated as a latent variable inferred from observed test responses and background characteristics. In the original analysis, this latent proficiency is modeled using a linear fixed-effects model.

We extend this analysis by replacing the fixed-effects model with a mixed-effects formulation that accounts for both student-level covariates and school-level variation. In addition, we replace the linear components of the model with neural networks, allowing for flexible nonlinear relationships.

\paragraph{\emph{Original PISA analysis.}}
Let students be indexed by $i$, schools by $j$, and items (test questions) by $k$.
Each item $k$ is assigned to one of three domains $t(k)$: reading, mathematics, and science, and has $C_k$ ordered score categories $\{0,\dots,C_k-1\}$.
Let $y_{ijk} \in \{0,\dots,C_k-1\}$ denote the ordinal response of student $i$ in school $j$ to item $k$.
Each student $i$ in school $j$ is modeled to have a latent proficiency vector $\varphi_{ij} \in \mathbb{R}^{3}$, where $\varphi_{ij,t(k)}$ is the proficiency in the domain of item $k$.
Each item $k$ is assigned a vector of step difficulties
$\delta_k = (\delta_{k,1},\dots,\delta_{k,C_k-1}) \in \mathbb{R}^{C_k-1}$.

In the PISA analysis, the probability of response $c$ is modeled with a Rasch partial credit measurement model,
\begin{equation}\label{eq:pisa_pcm}
\mathbb{P}(y_{ijk} = c \mid \varphi_{ij}, \delta_k)
=
\frac{
\exp\left(
\sum_{m=1}^{c}
\left(
\varphi_{ij,t(k)} - \delta_{k,m}
\right)
\right)
}{
\sum_{c'=0}^{C_k-1}
\exp\left(
\sum_{m=1}^{c'}
\left(
\varphi_{ij,t(k)} - \delta_{k,m}
\right)
\right)
},
\qquad c = 0,\dots,C_k-1,
\end{equation}
where item difficulties $\delta_k$ are treated as fixed and known (in our experiments we use the official PISA values).

Let $x_{ij}$ denote the high-dimensional background covariates for student $i$ in school $j$. These are assembled from student, parent, and school questionnaires together with core design variables.
In the official PISA analysis, these covariates are transformed with principal component analysis (PCA), denoted as $\tilde x_{ij}$.
The latent proficiency vector $\varphi_{ij}$ is then modeled by a multivariate normal regression
\begin{equation}\label{eq:pisa_background}
\varphi_{ij} \mid \tilde x_{ij} \sim \mathcal{N}\bigl(W \tilde x_{ij}, \Sigma_{\varphi}\bigr),
\end{equation}
where $W$ is a coefficient matrix and $\Sigma_{\varphi}$ is a full covariance matrix across domains.

\paragraph{\emph{Neural mixed-effects extension.}}
Our extension keeps the Rasch partial credit measurement model \eqref{eq:pisa_pcm} but modifies the background model \eqref{eq:pisa_background} in two ways: (i) we replace the linear predictor on PCA-transformed covariates $\tilde x_{ij}$ by a neural network on the raw background covariates $x_{ij}$, and (ii) we separate school-level covariates $z_j$ and introduce school-level random effects.
These yield the neural mixed-effects model
\begin{equation}\label{eq:pisa_neural_phi}
\varphi_{ij} = f_{\theta}(x_{ij}) + \gamma_j \odot g_{\psi}(z_j) +  \varepsilon_{ij}.
\end{equation}
where $\gamma_j \sim \mathcal{N}\bigl(0, I_3\bigr)$ and $\varepsilon_{ij} \sim \mathcal{N}\bigl(0, \sigma^2 I_3\bigr)$.
Here, $\odot$ denotes element-wise multiplication, $f_{\theta}(x_{ij}) \in \mathbb{R}^{3}$ maps student covariates to a domain-specific proficiency, $\gamma_j \odot g_{\psi}(z_j) \in \mathbb{R}^{3}$ produces domain-specific school-level effects, and
$\sigma > 0$ is a learned scalar.
Both $f_{\theta}$ and $g_{\psi}$ are neural networks.
The model is illustrated in Figure \ref{fig:pisa_illustration}. 

In this model, student proficiencies follow
\begin{equation}\label{eq:student_conditional}
\varphi_{ij} \mid \gamma_j
\sim
\mathcal{N}\Bigl(
f_{\theta}(x_{ij}) + \gamma_j \odot g_{\psi}(z_j),
\ \sigma^{2} I_{3}
\Bigr),
\end{equation}
and the resulting marginal likelihood is 
\begin{equation}\label{eq:marginal_likelihood} \small
 p(y \mid \theta,\psi,\delta) = \prod_{j=1}^{J}
\int
\left[
\prod_{i=1}^{n_j}
\int
\left(
\prod_{k \in Q(i,j)}
p\bigl(y_{ijk} \mid \varphi_{ij}, \delta_k\bigr)
\right)
p\bigl(\varphi_{ij} \mid \gamma_j, x_{ij}; \theta,\psi\bigr)
\;  d\varphi_{ij}
\right]
p(\gamma_j)
\; d\gamma_j,
\end{equation}
where $Q(i,j) \subseteq \{1,\dots,K\}$ is the set of items answered by student $i$ in school $j$.
We optimize this likelihood using the quadrature procedure described in Appendix \ref{sup:multi}.

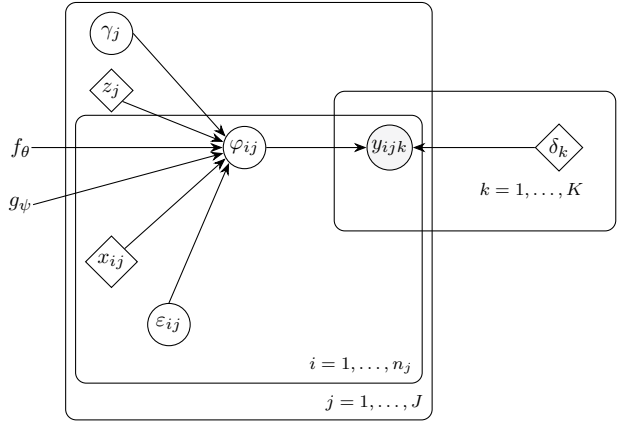
\begin{figure}
\centering
\begin{tikzpicture}[scale=0.8, transform shape]
\tikzset{>=Stealth}

\def\ports{0.95}    
\def\Lx{-2.2}       
\def\Fx{-3.7}       
\def\Dx{5.2}       
\def\Yx{2.4}       
\def\r{0.36}

\node[latent] (phi) at (0,0) {$\varphi_{ij}$};

\node[obs] (y) at (\Yx,0) {$y_{ijk}$};

\node[det] (delta) at (\Dx,0) {$\delta_k$};

\node[latent] (gammj) at (\Lx,  2*\ports) {$\gamma_j$};
\node[det]    (z)     at (\Lx,  1*\ports) {$z_j$};
\node[det]    (x)     at (\Lx, -2*\ports) {$x_{ij}$};
\node[latent] (eps) at (\Lx+0.95, -3.1*\ports) {$\varepsilon_{ij}$};

\node[const]  (theta) at (\Fx,  0*\ports) {$f_\theta$};
\node[const]  (psi)   at (\Fx, -1*\ports) {$g_\psi$};

\coordinate (p1) at ($(phi.center)+({\r*cos(150)},{\r*sin(150)})$);
\coordinate (p2) at ($(phi.center)+({\r*cos(165)},{\r*sin(165)})$);
\coordinate (p3) at ($(phi.center)+({\r*cos(180)},{\r*sin(180)})$);
\coordinate (p4) at ($(phi.center)+({\r*cos(195)},{\r*sin(195)})$);
\coordinate (p5) at ($(phi.center)+({\r*cos(210)},{\r*sin(210)})$);

\coordinate (p6) at ($(phi.center)+({\r*cos(225)},{\r*sin(225)})$);

\draw[->] (gammj.east) -- (p1);
\draw[->] (z.south east)     -- (p2);
\draw[->] (theta.east) -- (p3);   
\draw[->] (psi.east)   -- (p4);
\draw[->] (x.north east)     -- (p5);

\draw[->] (eps.north) -- (p6);

\draw[->] (phi.east) -- (y.west);
\draw[->] (delta.west) -- (y.east);

\plate[inner sep=12pt] {plateK} {(y)(delta)} {$k=1,\dots,K$};

\plate{plateI} {(x)(eps)(phi)(y)} {$i=1,\dots,n_j$};

\plate {plateJ} {(z)(gammj)(plateI)} {$j=1,\dots,J$};

\end{tikzpicture}
\caption{Illustration of the extended neural mixed-effects PISA model.}
\label{fig:pisa_illustration}
\end{figure}

\paragraph{\emph{Experimental setup.}}

We compare four background models coupled with the same partial credit measurement model in Equation \eqref{eq:pisa_pcm}: fixed-effects linear model (FE–LIN), random-effects linear model (ME–LIN), fixed-effects neural model (FE–NN), and the neural mixed-effects model (ME–NN). 
All models are fit by maximizing the marginal log-likelihood \eqref{eq:marginal_likelihood} for each country separately.
For each  model we evaluate predictive performance on (i) held-out students from observed schools and (ii) students in new schools. We report classification accuracy and mean negative log marginal likelihood for each proficiency domain.
Additional details are included in Appendix \ref{sup:pisa}.

\paragraph{\emph{Results.}}

We focus here on countries that were indicated in the original analysis as countries with the widest dispersion in scores of mathematical proficiency.Out of these, we report results for the 6 countries with the largest number of students.
Results for mathematical proficiency over 10 repetitions are reported in Figure \ref{fig:pisa_math}, and 
analogous results for proficiency in science and reading are reported in Appendix \ref{sup:pisa_results}.

\begin{figure} [ht]
    \centering
    \includegraphics[width=\linewidth]{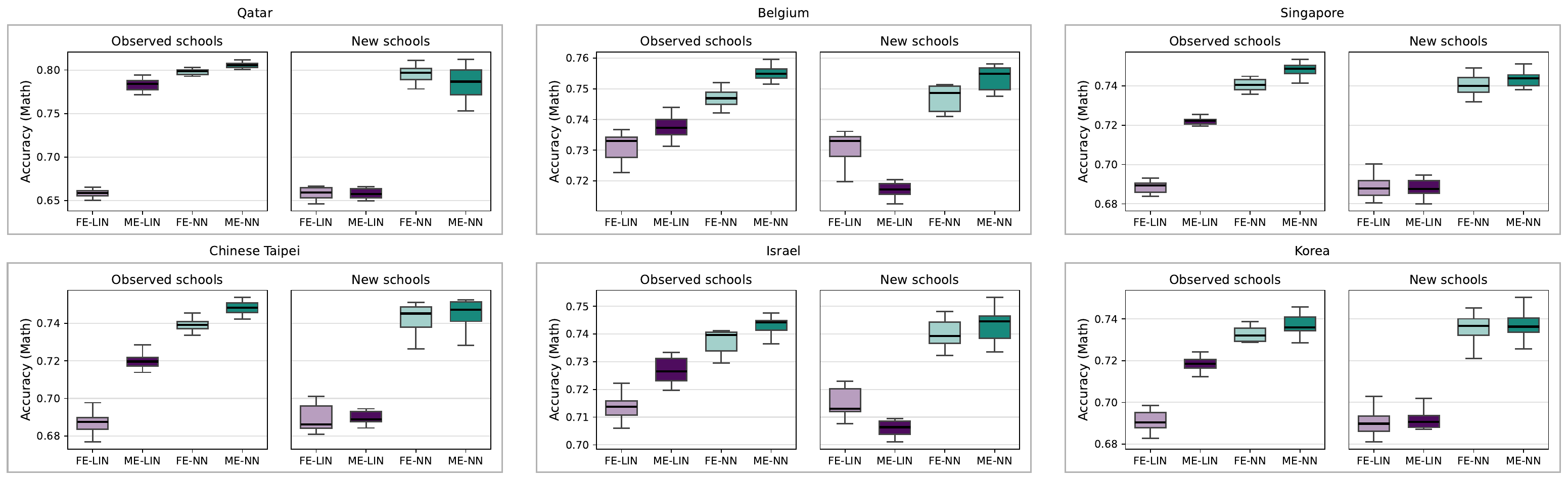}
    \caption{PISA math proficiency experimental results. In all six countries, neural models outperform their linear counterparts. For previously observed schools, incorporating random effects substantially improves performance in both linear and neural models. For new schools, the incremental benefit of adding random effects to the neural model is modest.}
    \label{fig:pisa_math}
\end{figure}

\section{Discussion} \label{sec:discuss}
In this work we introduced the neural generalized mixed-effects model (NGMM), a flexible extension of generalized linear mixed-effects models that replaces the linear fixed and random components with neural networks.

By reparameterizing the integral over random effects and expressing the marginal likelihood as an ordinary differential equation, we obtained a differentiable approximation of the marginal log-likelihood that is compatible with automatic differentiation and stochastic optimization. In the general case where the random-effects covariates vary across observations within a group, the one-dimensional representation is no longer available. We addressed this case in Appendix \ref{sup:multi}, where we introduce an additional approximation by replacing the resulting multivariate integral with a differentiable quadrature rule.
In both cases, the gap between the exact and approximate objectives decays at a Gaussian-tail rate.

This tailored approximation of the marginal likelihood can be substantially more accurate than generic methods such as black-box variational inference \citep{ranganath2014black}. We show this in Appendix \ref{sec:bbvi}.

For Gaussian and symmetric binary links, we showed that NGMM preserves the classical “new group” generalization property: the Bayes-optimal predictor for a previously unseen group depends only on the fixed-effects component. 

Our empirical studies show that these methodological contributions translate into practical gains across a range of settings. NGMM provides a competitive alternative to previous neural mixed-effects models that rely on closed-form solutions in conjugate settings, and extends seamlessly to non-conjugate cases.

Our experiments show that explicitly modeling random effects remains valuable even in high-dimensional neural architectures, particularly when generalization to new groups is of interest. Furthermore, the PISA case study illustrates how NGMM can be used in latent models.

Our results assume isotropic Gaussian random effects. This assumption permits a one-dimensional reduction of the random-effects integral when the random effects depend on group-level covariates, and simplifies the truncation analysis. Extending NGMM to more general random-effects distributions is a promising direction for future work. 

More broadly, NGMM links classical mixed-effects modeling with modern deep learning. It provides a general framework for incorporating group structure into neural models, thereby facilitating generalization from observed groups to new ones.
We expect this combination to be useful in many applied settings where data exhibits nonlinear relationships and is hierarchically organized. 

\section*{Acknowledgments}
This work is supported by NSF IIS-2127869, NSF DMS-2311108, ONR N000142412243, and the Simons Foundation.
YS is supported by a Founder’s Postdoctoral Fellowship, Department of Statistics, Columbia University.

\clearpage
\bibliographystyle{abbrvnat}
\bibliography{reference}

@article{zhou2012genome,
  title={Genome-wide efficient mixed-model analysis for association studies},
  author={Zhou, Xiang and Stephens, Matthew},
  journal={Nature Genetics},
  volume={44},
  number={7},
  pages={821--824},
  year={2012},
  publisher={Nature Publishing Group}
}

@article{listgarten2010correction,
  title={Correction for hidden confounders in the genetic analysis of gene expression},
  author={Listgarten, Jennifer and Kadie, Carl and Schadt, Eric E and Heckerman, David},
  journal={Proceedings of the National Academy of Sciences},
  volume={107},
  number={38},
  pages={16465--16470},
  year={2010},
  publisher={National Academy of Sciences}
}

@article{zhang2010mixed,
  title={Mixed linear model approach adapted for genome-wide association studies},
  author={Zhang, Zhiwu and Ersoz, Elhan and Lai, Chao-Qiang and Todhunter, Rory J and Tiwari, Hemant K and Gore, Michael A and Bradbury, Peter J and Yu, Jianming and Arnett, Donna K and Ordovas, Jose M and others},
  journal={Nature Genetics},
  volume={42},
  number={4},
  pages={355--360},
  year={2010},
  publisher={Nature Publishing Group}
}

@article{bernal2013statistical,
  title={Statistical analysis of longitudinal neuroimage data with linear mixed effects models},
  author={Bernal-Rusiel, Jorge L and Greve, Douglas N and Reuter, Martin and Fischl, Bruce and Sabuncu, Mert R and Alzheimer's Disease Neuroimaging Initiative and others},
  journal={NeuroImage},
  volume={66},
  pages={249--260},
  year={2013},
  publisher={Elsevier}
}

@article{visscher2003mixed,
  title={Mixed blocked/event-related designs separate transient and sustained activity in fMRI},
  author={Visscher, Kristina M and Miezin, Francis M and Kelly, James E and Buckner, Randy L and Donaldson, David I and McAvoy, Mark P and Bhalodia, Vidya M and Petersen, Steven E},
  journal={NeuroImage},
  volume={19},
  number={4},
  pages={1694--1708},
  year={2003},
  publisher={Elsevier}
}

@article{payne2015revisiting,
  title={Revisiting the incremental effects of context on word processing: Evidence from single-word event-related brain potentials},
  author={Payne, Brennan R and Lee, Chia-Lin and Federmeier, Kara D},
  journal={Psychophysiology},
  volume={52},
  number={11},
  pages={1456--1469},
  year={2015},
  publisher={Wiley Online Library}
}

@article{baayen2008mixed,
  title={Mixed-effects modeling with crossed random effects for subjects and items},
  author={Baayen, R Harald and Davidson, Douglas J and Bates, Douglas M},
  journal={Journal of Memory and Language},
  volume={59},
  number={4},
  pages={390--412},
  year={2008},
  publisher={Elsevier}
}

@article{gelman2007analysis,
  title={An analysis of the New York City police department's “stop-and-frisk” policy in the context of claims of racial bias},
  author={Gelman, Andrew and Fagan, Jeffrey and Kiss, Alex},
  journal={Journal of the American Statistical Association},
  volume={102},
  number={479},
  pages={813--823},
  year={2007},
  publisher={Taylor \& Francis}
}

@article{judd2012treating,
  title={Treating stimuli as a random factor in social psychology: a new and comprehensive solution to a pervasive but largely ignored problem.},
  author={Judd, Charles M and Westfall, Jacob and Kenny, David A},
  journal={Journal of Personality and Social Psychology},
  volume={103},
  number={1},
  pages={54},
  year={2012},
  publisher={American Psychological Association}
}

@article{van2023lecanemab,
  title={Lecanemab in early Alzheimer’s disease},
  author={Van Dyck, Christopher H and Swanson, Chad J and Aisen, Paul and Bateman, Randall J and Chen, Christopher and Gee, Michelle and Kanekiyo, Michio and Li, David and Reyderman, Larisa and Cohen, Sharon and others},
  journal={New England Journal of Medicine},
  volume={388},
  number={1},
  pages={9--21},
  year={2023},
  publisher={Massachusetts Medical Society}
}

@article{raudenbush1999synthesizing,
  title={Synthesizing results from the trial state assessment},
  author={Raudenbush, Stephen W and Fotiu, Randall P and Cheong, Yuk Fai},
  journal={Journal of Educational and Behavioral Statistics},
  volume={24},
  number={4},
  pages={413--438},
  year={1999},
  publisher={Sage Publications}
}

@article{nye2000effects,
  title={The effects of small classes on academic achievement: The results of the Tennessee class size experiment},
  author={Nye, Barbara and Hedges, Larry V and Konstantopoulos, Spyros},
  journal={American Educational Research Journal},
  volume={37},
  number={1},
  pages={123--151},
  year={2000},
  publisher={Sage Publications}
}

@article{lyu2023estimating,
  title={Estimating heterogeneous treatment effects within latent class multilevel models: A bayesian approach},
  author={Lyu, Weicong and Kim, Jee-Seon and Suk, Youmi},
  journal={Journal of Educational and Behavioral Statistics},
  volume={48},
  number={1},
  pages={3--36},
  year={2023},
  publisher={Sage Publications}
}

@book{bryk1992hierarchical,
  title={Hierarchical linear models: applications and data analysis methods},
  author={Bryk, Anthony S and Raudenbush, Stephen W},
  year={1992},
  publisher={Sage Publications, Inc}
}

@book{demidenko2013mixed,
  title={Mixed models: theory and applications with R},
  author={Demidenko, Eugene},
  year={2013},
  publisher={John Wiley \& Sons}
}

@book{generalized,
  title={Generalized Linear Models, 2nd ed.},
  author={McCullagh, Peter and Nelder, John},
  year={1989},
  publisher={Chapman and Hall},
  pages={452}
}

@book{McCullaghNelder1989GLM,
  author    = {McCullagh, P. and Nelder, J. A.},
  title     = {Generalized Linear Models},
  edition   = {2},
  year      = {1989},
  series    = {Monographs on Statistics and Applied Probability},
  number    = {37},
  publisher = {Chapman and Hall},
  address   = {London}
}

@inbook{Vaart_1998, 
  place={Cambridge}, 
  series={Cambridge Series in Statistical and Probabilistic Mathematics}, 
  title={M–and Z-Estimators}, 
  booktitle={Asymptotic Statistics}, 
  publisher={Cambridge University Press}, 
  author={Vaart, A. W. van der}, 
  year={1998}, 
  pages={41–84}, 
  collection={Cambridge Series in Statistical and Probabilistic Mathematics},
  chapter={5}
}

@inproceedings{MeNets,
  title={Mixed effects neural networks (menets) with applications to gaze estimation},
  author={Xiong, Yunyang and Kim, Hyunwoo J and Singh, Vikas},
  booktitle={Proceedings of the IEEE/CVF Conference on Computer Vision and Pattern Recognition},
  pages={7743--7752},
  year={2019}
}

@article{DeepGLMM,
  title={Bayesian deep net GLM and GLMM},
  author={Tran, M-N and Nguyen, Nghia and Nott, David and Kohn, Robert},
  journal={Journal of Computational and Graphical Statistics},
  volume={29},
  number={1},
  pages={97--113},
  year={2020},
  publisher={Taylor \& Francis}
}

@article{Simchoni_Rosset1,
  title={Using random effects to account for high-cardinality categorical features and repeated measures in deep neural networks},
  author={Simchoni, Giora and Rosset, Saharon},
  journal={Advances in Neural Information Processing Systems},
  volume={34},
  pages={25111--25122},
  year={2021}
}

@article{Simchoni_Rosset2,
  author  = {Giora Simchoni and Saharon Rosset},
  title   = {Integrating Random Effects in Deep Neural Networks},
  journal = {Journal of Machine Learning Research},
  year    = {2023},
  volume  = {24},
  number  = {156},
  pages   = {1--57}
}

@inproceedings{pmlr-v202-lee23k,
  title={H-Likelihood Approach to Deep Neural Networks with Temporal-Spatial Random Effects for High-Cardinality Categorical Features},
  author={Lee, Hangbin and Lee, Youngjo},
  booktitle={Proceedings of the 40th International Conference on Machine Learning},
  pages={18974--18987},
  year={2023},
  editor={Krause, Andreas and Brunskill, Emma and Cho, Kyunghyun and Engelhardt, Barbara and Sabato, Sivan and Scarlett, Jonathan},
  volume={202},
  series={Proceedings of Machine Learning Research},
  month={23--29 Jul},
  publisher={PMLR}
}

@article{lee2023subject,
  title={Subject-specific Deep Neural Networks for Count Data with High-cardinality Categorical Features},
  author={Lee, Hangbin and Ha, Il Do and Hwang, Changha and Lee, Youngjo},
  journal={arXiv preprint arXiv:2310.11654},
  year={2023}
}

@inproceedings{tschalzev2024enabling,
  title={Enabling mixed effects neural networks for diverse, clustered data using Monte Carlo Methods},
  author={Tschalzev, Andrej and Nitschke, Paul and Kirchdorfer, Lukas and L{\"u}dtke, Stefan and Bartelt, Christian and Stuckenschmidt, Heiner},
  booktitle={Proceedings of the Thirty-Third International Joint Conference on Artificial Intelligence},
  pages={5018--5024},
  year={2024}
}

@article{kilian2023mixed,
  title={Mixed effects in machine learning--A flexible mixedML framework to add random effects to supervised machine learning regression},
  author={Kilian, Pascal and Ye, Sangbeak and Kelava, Augustin},
  journal={Transactions on Machine Learning Research},
  year={2023}
}

@article{mandel2023neural,
  title={Neural networks for clustered and longitudinal data using mixed effects models},
  author={Mandel, Francesca and Ghosh, Riddhi Pratim and Barnett, Ian},
  journal={Biometrics},
  volume={79},
  number={2},
  pages={711--721},
  year={2023},
  publisher={Wiley Online Library}
}

@book{van2000asymptotic,
  title={Asymptotic statistics},
  author={Van der Vaart, Aad W},
  volume={3},
  year={2000},
  pages={25-29, 46-47},
  publisher={Cambridge University Press}
}

@inproceedings{apostol1958mathematical,
  title={Mathematical analysis},
  author={Apostol, Tom M and Ablow, CM},
  year={1958},
  publisher={American Institute of Physics},
  volume={2},
  pages={276-277}
}

@book{hairer1993solving,
  title={Solving ordinary differential equations I: Nonstiff problems},
  author={Hairer, Ernst and Wanner, Gerhard and N{\o}rsett, Syvert P},
  year={1993},
  pages={162},
  publisher={Springer}
}

@techreport{pisa,
  author={{OECD}},
  title={{PISA 2012 Technical Report}},
  year={2014},
  institution={{OECD Publishing}},
  address={Paris},
  pages={472},
  doi={10.1787/6341a959-en},
  url={https://doi.org/10.1787/6341a959-en},
  note={Programme for International Student Assessment (PISA)}
}

@inproceedings{ranganath2014black,
  title={Black box variational inference},
  author={Ranganath, Rajesh and Gerrish, Sean and Blei, David},
  booktitle={Proceedings of the Seventeenth International Conference on Artificial Intelligence and Statistics},
  pages={814--822},
  year={2014},
  publisher={PMLR}
}

@inproceedings{taylor2019rxrx1,
  author={Taylor, J. and Earnshaw, B. and Mabey, B. and Victors, M. and Yosinski, J.},
  title={RxRx1: An Image Set for Cellular Morphological Variation Across Many Experimental Batches},
  year={2019},
  booktitle={Proceedings of the International Conference on Learning Representations},
  note={AI for Social Good Workshop}
}

\clearpage

\begin{appendices}

\section{Instantiation for discrete outcomes: random-intercept logistic model} \label{sup:logistic}

As in Section \ref{sec:obj}, here we focus on a single group and drop the group index, writing $(x,y,z,\gamma)$ instead of $(\xmat_j,\yvec_j,\zmat_j,\gamma_j)$, where $y\in\{0,1\}^{n_j}$ and $x_i$ denotes the $i$-th row of $x$.

We illustrate the one–dimensional ODE construction in the special case of a random-intercept logistic model.
In this case, conditionally on $(x_i,\mathrm{z},\gamma)$, the response $y_i\in\{0,1\}$ satisfies
\begin{equation}
p(y_i=1 \mid x_i,\mathrm{z},\gamma)
=
\sigma\bigl(f_\theta(x_i) + g_\psi(\mathrm{z})\gamma\bigr),
\end{equation}
where $\gamma \sim \mathcal{N}(0,1)$, $g_\psi(\mathrm{z})\in\R$, and
\begin{equation}
\sigma(u) = \frac{1}{1 + e^{-u}}
\end{equation}
is the logistic function.
Conditionally on $(x,\mathrm{z},\gamma)$, the entries of $y$ are independent, so
\begin{equation}
p(y \mid x,z,\gamma)
=
\prod_{i=1}^{n_j}
\sigma\bigl(f_\theta(x_i) + g_\psi(\mathrm{z})\gamma\bigr)^{y_i}
\Bigl(1-\sigma\bigl(f_\theta(x_i) + g_\psi(\mathrm{z})\gamma\bigr)\Bigr)^{1-y_i}.
\end{equation}
The group-level marginal likelihood is therefore
\begin{equation}
p(y \mid x,z)
=
\int_{\R}
p(y \mid x,z,\gamma)\,\phi_1(\gamma)\, d\gamma.
\end{equation}

The corresponding change of variables is
\begin{equation}
\gamma = \operatorname{logit}(\xi), \qquad \xi \in (0,1),
\end{equation}
where
\begin{equation}
\operatorname{logit}(\xi) = \log \frac{\xi}{1-\xi},
\end{equation}
with derivative
\begin{equation}
\frac{d\gamma}{d\xi}
=
\frac{1}{\xi(1-\xi)}.
\end{equation}
This is an exact reparameterization, and the marginal likelihood becomes
\begin{equation}
p(y \mid x,z)
=
\int_{0}^{1}
p\bigl(y \mid x,z,\operatorname{logit}(\xi)\bigr)\,
\phi_1\bigl(\operatorname{logit}(\xi)\bigr)\,
\frac{1}{\xi(1-\xi)}
\, d\xi.
\end{equation}

Then,
\begin{equation}
p(y \mid x,z)
=
\int_{0}^{1} \exp\bigl(L(\xi; y,x,z)\bigr)\, d\xi,
\end{equation}
where the log-integrand is
\begin{equation}
L(\xi; y,x,z)
\coloneqq
\log p\bigl(y \mid x,z,\operatorname{logit}(\xi)\bigr)
+
\log \phi_1\bigl(\operatorname{logit}(\xi)\bigr)
-
\log\bigl(\xi(1-\xi)\bigr).
\end{equation}
For numerical stability, truncate the interval to $[\epsilon,1-\epsilon]$ with a small constant $\epsilon>0$ and approximate
\begin{equation}
\log p(y \mid x,z)
\approx
\log \int_{\epsilon}^{1-\epsilon} \exp\bigl(L(\xi; y,x,z)\bigr)\, d\xi.
\end{equation}

We approximate the initial condition at $\xi = \epsilon$ by starting from a finite baseline constant $H_0$:
\begin{equation}
\frac{dH(\xi)}{d\xi}
=
\exp\bigl(L(\xi; y,x,z) - H(\xi)\bigr),
\quad
H(\epsilon) = H_0.
\end{equation}
Integrating the ODE over $\xi \in [\epsilon, 1-\epsilon]$ yields
\begin{equation}
H(1-\epsilon)
= 
\log \left( \exp(H(\epsilon)) + \int_{\epsilon}^{1-\epsilon}  \exp\bigl(L(\xi; y,x,z)\bigr)\, d\xi\right) , 
\end{equation}
since $\lim_{\epsilon \to 0^+} \exp(H(\epsilon)) = 0$, for sufficiently small $\epsilon$ we have 
\begin{equation}
\log p(y \mid x,z)
\approx
H(1-\epsilon) \approx
\log \int_{\epsilon}^{1-\epsilon}  \exp\bigl(L(\xi; y,x,z)\bigr)\, d\xi,
\end{equation}
providing a differentiable objective for learning the network parameters $\theta$ and $\psi$.

\section{Explicit integral expressions}
\label{sec:explicitntegrals}
Here we provide explicit expressions for the error rates and error bounds in the Gaussian and Probit models, as well as details on the numerical experiments used to illustrate Theorem \ref{thm:bound}.

\medskip
\noindent\textit{Notational remark.}
Throughout this appendix we fix $(x,\mathrm{z},\theta,\psi)$ and suppress group dependence in the notation.
We reuse the notation $I(y)$ and $I_M(y)$ for the exact and truncated marginal integrals of the \emph{full} one-dimensional integrand, namely
\begin{equation}
\left(\prod_{i=1}^{n_j} h(y_i)\right)
\exp\left(
\sum_{i=1}^{n_j}\left[y_i\,\vartheta(x_i,\mathrm{z},\xi;\theta,\psi)-A\bigl(\vartheta(x_i,\mathrm{z},\xi;\theta,\psi)\bigr)\right]
\right)\phi_1(\xi),
\end{equation}
which differs from the definitions in Section \ref{sec:theory} by the factor $\prod_{i=1}^{n_j} h(y_i)$, independent of $\xi$.
Thus, the truncation error $\Delta_M(y)=\log I(y)-\log I_M(y)$ remains the same under both conventions.

\subsection{Gaussian model}

While the expressions for $I(y)$ and $I_{M}(y)$ below hold for arbitrary $\sigma^{2} > 0$, here we focus on $\sigma^{2} = 1$.

\begin{equation}
I_{M}(y) \coloneqq \int_{-M}^{M} \left(\prod_{i=1}^{n_j} \mathcal{N}(y_i;f_i + r\xi,\sigma^{2})\right)\phi_1(\xi) d\xi.
\end{equation}
where $f_i \coloneqq f_\theta(x_i)$ and $r \coloneqq \Vert g_\psi(\mathrm{z})\Vert$.
Therefore, the posterior $p(\xi \mid y)$ is univariate Gaussian with
\begin{equation}
\mu \coloneqq \E[\xi \mid y] = \frac{r}{\sigma^{2} + n_j r^{2}}\sum_{i=1}^{n_j}(y_i-f_i),
\qquad
s^{2} \coloneqq \mathrm{Var}(\xi \mid y) = \frac{\sigma^{2}}{\sigma^{2} + n_j r^{2}},
\end{equation}
and
\begin{equation}
I_{M}(y) = I(y)\,p\bigl(\vert \xi \vert \le M \mid y\bigr)
= I(y) \, \left[\Phi_1\left(\frac{M-\mu}{s}\right) - \Phi_1\left(\frac{-M-\mu}{s}\right)\right].
\end{equation}
The truncation error is
\begin{equation}
\Delta_{M}(y) \coloneqq \log I(y) - \log I_{M}(y)
= -\log p\bigl(\vert \xi \vert \le M \mid y\bigr).
\end{equation}

\subsubsection{Exact truncation error rate for the Gaussian model}
In the Gaussian model, 
\begin{equation}
L(\xi)
=
\sum_{i=1}^{n_j}\left(y_i(f_i + r\xi) - \frac{1}{2}\bigl(f_i + r\xi\bigr)^{2} - \frac{1}{2}y_i^{2}\right)
-
\frac{1}{2}\xi^{2}
-
\frac{n_j+1}{2}\log(2\pi).
\end{equation}
and by rearranging we get
\begin{equation} \notag
L(\xi)
= -\frac{1 + n_j r^{2}}{2}\,\xi^{2} + r\left(\sum_{i=1}^{n_j}(y_i-f_i)\right)\xi
+ \left(\sum_{i=1}^{n_j} y_i f_i - \frac{1}{2}\sum_{i=1}^{n_j} f_i^{2} - \frac{1}{2}\sum_{i=1}^{n_j} y_i^{2} - \frac{n_j+1}{2}\log(2\pi)\right).
\end{equation}

Define
\begin{equation}
C_{\mathrm{G}} \coloneqq \frac{1 + n_j r^{2}}{2},
\qquad
m \coloneqq \frac{r\sum_{i=1}^{n_j}(y_i-f_i)}{1 + n_j r^{2}}.
\end{equation}
Completing the square, we can write
\begin{equation}
L(\xi)
= -C_{\mathrm{G}}(\xi - m)^{2} + d,
\end{equation}
for a constant $d = d(y,f,r)$ independent of $\xi$. Hence,
\begin{equation}
\exp\bigl(L(\xi)\bigr)
= \exp\bigl(d\bigr)\,\exp\bigl(-C_{\mathrm{G}}(\xi - m)^{2}\bigr).
\end{equation}

Let $\tau^{2} \coloneqq 1/(1 + n_j r^{2}) = 1/(2C_{\mathrm{G}})$ and let $Z$ be a normal random
variable with mean $m$ and variance $\tau^{2}$, with density
\begin{equation}
\varphi_{m,\tau}(\xi)
\coloneqq
\frac{1}{\sqrt{2\pi}\,\tau}
\exp\Bigl(-\frac{(\xi - m)^{2}}{2\tau^{2}}\Bigr).
\end{equation}
Then
\begin{equation}
\exp\bigl(-C_{\mathrm{G}}(\xi - m)^{2}\bigr)
=
\sqrt{2\pi}\,\tau\,\varphi_{m,\tau}(\xi),
\end{equation}
and therefore the tail integral can be written as
\begin{equation}
R_{M}(y) \coloneqq \int_{\vert\xi\vert > M} \exp\bigl(L(\xi)\bigr) d\xi
=
K(y,f,r)\, \mathbb P\bigl(\vert Z \vert > M\bigr),
\end{equation}
where
\begin{equation}
K(y,f,r) \coloneqq \exp\bigl(d(y,f,r)\bigr)\,\sqrt{2\pi}\,\tau.
\end{equation}

For a normal random variable $Z$ with mean $m$ and variance $\tau^{2}$, the Gaussian
tail satisfies
\begin{equation}
\lim_{M\to\infty}
\frac{1}{M^{2}} \log \mathbb P\bigl(\vert Z \vert > M\bigr)
=
-\frac{1}{2\tau^{2}}
=
- C_{\mathrm{G}}.
\end{equation}
Consequently,
\begin{equation}
\lim_{M\to\infty}
\frac{1}{M^{2}} \log R_{M}(y)
=
- C_{\mathrm{G}}.
\end{equation}

As $M\to\infty$ we have $R_{M}(y)\to 0$ and $I_{M}(y)\to I(y) > 0$, so
\begin{equation}
\Delta_{M}(y)
=
\log\Bigl(1 + \frac{R_{M}(y)}{I_{M}(y)}\Bigr),
\end{equation}
and $\frac{R_{M}(y)}{I_{M}(y)} \to 0$, $\frac{I_{M}(y)}{I(y)} \to 1$ imply
\begin{equation}
\lim_{M\to\infty}
\frac{\Delta_{M}(y)}{R_{M}(y)}
=
\frac{1}{I(y)}.
\end{equation}
Combining the last two equations yields
\begin{equation}
\lim_{M\to\infty}
\frac{1}{M^{2}} \log \Delta_{M}(y)
=
- C_{\mathrm{G}}
=
-\frac{1 + n_j r^{2}}{2}.
\end{equation}
In particular, in the canonical Gaussian experiment with $r = 1$ we have
$C_{\mathrm{G}} = (1+n_j)/2$ and the asymptotic slope is $-(1+n_j)/2$.

\subsubsection{Explicit truncation error bound for the Gaussian model}

Define
\begin{equation}
\Delta_{M}^{\mathrm{bound}}(y)
\coloneqq
\frac{C(y;x,z,\theta,\psi)}{M}
\exp\Bigl(-\frac{M^{2}}{2}\Bigr),
\end{equation}
so that $\Delta_{M}(y) \le \Delta_{M}^{\mathrm{bound}}(y)$ for all $M \ge M_{0}(y;x,z,\theta,\psi)$.
From the explicit form of $\Delta_{M}^{\mathrm{bound}}(y)$ we have
\begin{equation}
\lim_{M\to\infty}
\frac{1}{M^{2}} \log \Delta_{M}^{\mathrm{bound}}(y)
=
-\frac{1}{2},
\end{equation}
whereas the exact asymptotic result above gives
\begin{equation}
\lim_{M\to\infty}
\frac{1}{M^{2}} \log \Delta_{M}(y)
=
- C_{\mathrm{G}}
=
-\frac{1 + n_j r^{2}}{2}.
\end{equation}
Therefore,
\begin{equation}
\lim_{M\to\infty}
\frac{1}{M^{2}}
\log\Bigl(
\frac{\Delta_{M}^{\mathrm{bound}}(y)}{\Delta_{M}(y)}
\Bigr)
=
\Bigl(-\frac{1}{2}\Bigr) - \bigl(-C_{\mathrm{G}}\bigr)
=
C_{\mathrm{G}} - \frac{1}{2}
=
\frac{n_j r^{2}}{2}
\ge 0.
\end{equation}
In particular, in the canonical Gaussian example with $r=1$ we obtain
\begin{equation}
\lim_{M\to\infty}
\frac{1}{M^{2}}
\log\Bigl(
\frac{\Delta_{M}^{\mathrm{bound}}(y)}{\Delta_{M}(y)}
\Bigr)
=
\frac{n_j}{2},
\end{equation}
so on the logarithmic scale the gap between the true error and the upper bound grows quadratically in $M$ with slope $n_j/2$.

\subsubsection{Numerical evaluation}

For the experiment in Figure \ref{fig:truncation} we set $\sigma^{2}=1$ and vary the values of $\{f_i\}_{i=1}^{n_j}$, $r$, and $y$.
For fixed $r\neq0$, different choices of $(\{f_i\}_{i=1}^{n_j},y)$ do not affect the leading quadratic tail exponent $C_{\mathrm G}$, and change only lower-order terms.
The truncation levels $M$ are chosen on a uniform grid.

To visualize the theoretical bound from Theorem \ref{thm:bound}, we plot
\begin{equation}
M \mapsto \frac{1}{M}\exp\Bigl(-\frac{M^{2}}{2}\Bigr)
\end{equation}
rescaled by a constant factor so that it matches the exact truncation error at a moderate reference value of $M$.
This rescaling does not affect the tail exponent and allows a direct comparison between the true rate $\exp(-C_{\mathrm{G}}M^{2})$ and the generic bound $\exp(-M^{2}/2)$ in Figure \ref{fig:truncation}.

\subsection{Probit model}
In the Probit model, conditionally on $\xi$, the entries of $y$ are independent and satisfy
\begin{equation}
y_i \mid \xi \sim \mathrm{Bernoulli}\bigl(\Phi(f_i + r\xi)\bigr),
\qquad i=1,\dots,n_j,
\end{equation}
where $f_i \coloneqq f_\theta(x_i)$ and $r \coloneqq \Vert g_\psi(\mathrm{z})\Vert$.
Equivalently,
\begin{equation}
p(y \mid f + r\xi)
=
\prod_{i=1}^{n_j}
\Phi(f_i + r\xi)^{y_i}\,\Phi(-f_i - r\xi)^{1-y_i}.
\end{equation}
The marginal integral is
\begin{equation}
I(y) \coloneqq \int_{\R} p(y \mid f + r\xi)\,\phi_1(\xi)\, d\xi,
\end{equation}
and for $M>0$, the truncated integrals are
\begin{equation}
I_{M}(y) \coloneqq \int_{-M}^{M} p(y \mid f + r\xi)\,\phi_1(\xi)\, d\xi.
\end{equation}

\subsubsection{Exact truncation error rate for the Probit model}

In the Probit model the log–integrand is
\begin{equation}
L(\xi)
= \log p\bigl(y \mid f + r\xi\bigr) + \log \phi_1(\xi),
\end{equation}
and the tail integral is
\begin{equation}
R_{M}(y) \coloneqq I(y) - I_{M}(y)
= \int_{\vert\xi\vert > M} p\bigl(y \mid f + r\xi\bigr)\phi_1(\xi) \, d\xi.
\end{equation}
Let $Z \sim \mathcal{N}(0,1)$ and denote
\begin{equation}
g_{y}(\xi) \coloneqq p\bigl(y \mid f + r\xi\bigr).
\end{equation}
Then
\begin{equation}
R_{M}(y)
=
\int_{\vert\xi\vert > M} g_{y}(\xi)\phi_1(\xi) \, d\xi
=
\E\bigl[g_{y}(Z) \, \mathbbm{1}\{\vert Z \vert > M\}\bigr]
=
p\bigl(\vert Z \vert > M\bigr) \,\E\bigl[g_{y}(Z)\mid \vert Z \vert > M\bigr].
\end{equation}

Let
\begin{equation}
n_1 \coloneqq \sum_{i=1}^{n_j} y_i,
\qquad
n_0 \coloneqq n_j - n_1.
\end{equation}
As $\xi\to+\infty$, each factor $\Phi(f_i+r\xi)^{y_i}\to 1$, whereas each factor
$\Phi(-f_i-r\xi)^{1-y_i}$ decays at a Gaussian tail rate; consequently the $+\infty$
tail is governed by $n_0$.
Similarly, as $\xi\to-\infty$, the $-\infty$ tail is governed by $n_1$.
In particular, the two tails have Gaussian exponents
\begin{equation}
C_{+}(y) \coloneqq \frac{1+n_0 r^{2}}{2},
\qquad
C_{-}(y) \coloneqq \frac{1+n_1 r^{2}}{2},
\qquad
C_{\mathrm{P}}(y) \coloneqq \min\{C_{+}(y),C_{-}(y)\}.
\end{equation}
Then
\begin{equation}
\lim_{M\to\infty}\frac{1}{M^{2}}\log R_{M}(y)
=
- C_{\mathrm{P}}(y).
\end{equation}

In particular, $R_M(y)\to 0$ and $I_M(y)\to I(y)>0$ as $M\to\infty$, so
\begin{equation}
\Delta_{M}(y)
\coloneqq
\log I(y) - \log I_{M}(y)
=
\log\Bigl(1 + \frac{R_{M}(y)}{I_{M}(y)}\Bigr),
\end{equation}
and therefore
\begin{equation}
\lim_{M\to\infty}\frac{1}{M^{2}}\log \Delta_{M}(y)
=
- C_{\mathrm{P}}(y).
\end{equation}
Since $C_{\mathrm{P}}(y)\ge \tfrac{1}{2}$ with equality if and only if $n_0=0$ or $n_1=0$, the generic bound from Theorem~\ref{thm:bound} has the correct worst-case Gaussian tail exponent $\tfrac{1}{2}$, while for mixed response vectors (both $0$ and $1$ present) the truncation error decays faster.

\subsubsection{Numerical evaluation}
The exact marginal $I(y)$ and the truncated integrals $I_{M}(y)$ are computed by numerical quadrature of the one-dimensional integrals above.
The truncation levels $M$ are chosen on the same grid as in the Gaussian experiment.

To visualise the theoretical bound from Theorem~\ref{thm:bound}, we overlay
\begin{equation}
M \mapsto \frac{1}{M}\exp\Bigl(-\frac{M^{2}}{2}\Bigr),
\end{equation}
rescaled by a constant factor so that it matches the exact truncation error $\Delta_{M}(y)$ at a moderate reference value of $M$.
This rescaling does not affect the Gaussian tail exponent $1/2$ in $M^{2}$, and allows a direct comparison between $\Delta_{M}(y)$ and the generic bound from Theorem~\ref{thm:bound} in Figure~\ref{fig:truncation}.

\section{Analysis of numerical errors} \label{sup:numerical}

Recall that, for a fixed truncation radius $M = \frac{1}{2} (\xi_1 - \xi_0)>0$, our estimation procedure targets the truncated empirical objective
\begin{equation}
\bar{\ell}_M(\theta,\psi)
\coloneqq
\frac{1}{m}\sum_{j=1}^{m}
\Bigl[
\sum_{i=1}^{n_j}\log h(y_{ij})
+
H(\xi_{1};\yvec_j,\xmat_j,\zmat_j,\theta,\psi)
\Bigr],
\end{equation}
but in practice replaces $H$ by its numerical approximation $\Hnum$ and optimizes the ODE-based objective
\begin{equation}
\tilde{\ell}_M(\theta,\psi)
\coloneqq
\frac{1}{m}\sum_{j=1}^{m}
\Bigl[
\sum_{i=1}^{n_j}\log h(y_{ij})
+
\Hnum(\xi_{1};\yvec_j,\xmat_j,\zmat_j,\theta,\psi)
\Bigr].
\end{equation}
Moreover, we maximize $\tilde{\ell}_M(\theta,\psi)$ by stochastic gradient ascent:
let $\alpha_t = (\theta_t,\psi_t)$ denote the parameters at iteration $t$, and let
$\mathcal{B}_t \subset \{1,\dots,m\}$ be a mini-batch of size $\vert \mathcal{B}_t \vert = b$,
sampled uniformly at random from the $m$ groups.
The SGD update takes the form
\begin{equation}
\alpha_{t+1}
=
\alpha_t
+
\zeta_t g_t,
\end{equation}
where $\zeta_t>0$ is a step-size and
\begin{equation}
g_t
=
\frac{1}{\vert \mathcal{B}_t \vert}
\sum_{j\in\mathcal{B}_t}
\nabla_{\alpha}
\Bigl(
\sum_{i=1}^{n_j}\log h(y_{ij})
+
\Hnum(\xi_{1};\yvec_j,\xmat_j,\zmat_j,\alpha_t)
\Bigr)
\end{equation}
is the mini-batch stochastic gradient at $\alpha_t$.

In our setting, both the forward quantity $H(\xi;y,x,z,\theta,\psi)$ and the adjoint sensitivities
$\nabla_{\alpha} H(\xi;y,x,z,\theta,\psi)$ are defined as solutions of smooth ODEs on the finite interval $[\xi_0,\xi_1]$, and we approximate them with standard adaptive ODE solvers.
Classical numerical analysis for one-step methods (see, e.g., Theorem 3.6 and equation 3.19 in Chapter II.3 of \citealp{hairer1993solving} for convergence and global error bounds, and Section II.8 for asymptotic expansions)
implies that, under smoothness and Lipschitz conditions, for any compact parameter set $\ParamSet$ and any bounded subset $D$ of the data space, there exist constants $C_{\mathrm{val}},C_{\mathrm{grad}}<\infty$ such that for all $(y,x,z)\in D$,
\begin{equation}
\sup_{\alpha\in\ParamSet}
\bigl\vert
\Hnum(\xi_{1};y,x,z,\alpha)
-
H(\xi_{1};y,x,z,\alpha)
\bigr\vert
\le
\tau \, C_{\mathrm{val}},
\end{equation}
and
\begin{equation}
\sup_{\alpha\in\ParamSet}
\bigl\Vert
\nabla_{\alpha}
\Hnum(\xi_{1};y,x,z,\alpha)
-
\nabla_{\alpha}
H(\xi_{1};y,x,z,\alpha)
\bigr\Vert
\le
\tau \, C_{\mathrm{grad}},
\end{equation}
where $\tau>0$ is the solver tolerance (or a bound proportional to the stepsize).
Substituting these bounds into the mini-batch objective corresponding to $\bar{\ell}_M$ and $\tilde{\ell}_M$ respectively shows that for the same constants and any compact $\ParamSet$,
\begin{equation}
\sup_{\alpha\in\ParamSet}
\bigl\vert
\tilde{\ell}_M^{\mathrm{mb}}(\alpha) - \bar{\ell}_M^{\mathrm{mb}}(\alpha)
\bigr\vert
\le
C_{\mathrm{val}} \,\tau,
\qquad
\sup_{\alpha\in\ParamSet}
\bigl\Vert
\nabla_{\alpha}\tilde{\ell}_M^{\mathrm{mb}}(\alpha)
-
\nabla_{\alpha}\bar{\ell}_M^{\mathrm{mb}}(\alpha)
\bigr\Vert
\le
C_{\mathrm{grad}} \,\tau.
\end{equation}

Thus, fixing the forward and adjoint solvers and their tolerances defines a deterministic discretized truncated objective $\tilde{\ell}_M^{\mathrm{mb}}(\theta,\psi)$ whose values and gradients are uniformly close to those of $\bar{\ell}_M^{\mathrm{mb}}(\theta,\psi)$.
Stochastic gradient ascent is therefore optimizing this fixed discretized objective, and the effect of numerical ODE error is fully captured by the tolerance-dependent perturbation bounds above, rather than by an accumulation of error across iterations.

In our model, $f_{\theta}$ and $g_{\psi}$ are neural networks.
For any compact parameter set and any bounded subset of the data space, neural networks with smooth activation functions (e.g. $\tanh$ or softplus) are continuously differentiable and globally Lipschitz in $(x,z,\alpha)$ on that set.
If ReLU activations are used instead, the networks remain locally Lipschitz and piecewise $C^{1}$, so the forward and adjoint ODEs still satisfy the required Lipschitz conditions.
Formally, the numerical error bounds above apply on any compact subset of parameter space; in practice, optimization (or standard regularization such as weight decay) keeps the iterates in a bounded region, so the analysis applies on the subset of $\ParamSet$ visited in practice.

\section{Proofs}

\subsection{Proof of Lemma \ref{lemma:int}} \label{sec:proof_lemma_int}

\begin{proof}

For $g_\psi(\mathrm{z})=0$ the identity holds trivially as both sides equal $p(y\mid x,0)$. 
Therefore, assume $g_\psi(\mathrm{z})\neq0$ so that
\begin{equation}
    p\bigl(y \mid x, z \bigr) = \int_{\R^k} p\bigl(y \mid x, g_\psi(\mathrm{z})^{\intercal}\gamma\bigr)\,\phi_k(\gamma)\,d\gamma.   
\end{equation}
Denote
\begin{equation}
   \pmb u \coloneqq \frac{g_\psi(\mathrm{z})}{\Vert g_\psi(\mathrm{z})\Vert}.
\end{equation}
Extend $\pmb u$ to an orthonormal basis $\{\pmb u,\pmb w_1,\dots,\pmb w_{k-1}\}$ of $\R^k$, and denote by 
\begin{equation}
\pmb W \coloneqq
\begin{pmatrix}
\pmb u & \pmb w_1 & \dots & \pmb w_{k-1}
\end{pmatrix}
\in \R^{k\times k}
\end{equation}
the resulting orthogonal matrix, so that $\pmb W^{\intercal}\pmb W = I_k$.

We perform the change of variables
\begin{equation}
s \coloneqq \pmb W^{\intercal}\gamma.
\end{equation}
Since $\pmb W^{\intercal}$ is orthogonal, we have $\gamma = \pmb W s$ and the Jacobian determinant of this transformation is $1$, so $d\gamma = ds$. Moreover,
\begin{equation}
g_\psi(\mathrm{z})^{\intercal}\gamma =
g_\psi(\mathrm{z})^{\intercal}\pmb Ws = 
(\pmb W^{\intercal} g_\psi(\mathrm{z}))^{\intercal} s.
\end{equation}
By construction, $\pmb W^{\intercal} g_\psi(\mathrm{z}) = \Vert g_\psi(\mathrm{z})\Vert e_1$, where $e_1$ is the first standard basis vector in $\R^k$. Hence
\begin{equation}
g_\psi(\mathrm{z})^{\intercal}\gamma
=
\Vert g_\psi(\mathrm{z})\Vert\,s_1.
\end{equation}

Under $\gamma\sim\mathcal{N}(0,I_k)$ and the orthogonal change of variables, we have $s\sim\mathcal{N}(0,I_k)$ and
\begin{equation}
\phi_k(\gamma)\,d\gamma
=
\phi_k(s)\,ds.
\end{equation}
Therefore,
\begin{align}
& \int_{\R^k} p\bigl(y \mid x, g_\psi(\mathrm{z})^{\intercal}\gamma\bigr)\,\phi_k(\gamma)\,d\gamma
=
\int_{\R^k} p\bigl(y \mid x,\Vert g_{\psi}(\mathrm{z})\Vert s_1\bigr)\,\phi_k(s)\,ds \\
&=
\int_{\R^k} \left(p\bigl(y \mid x,\Vert g_{\psi}(\mathrm{z})\Vert s_1\bigr)\,\prod_{\ell=1}^k \phi_1(s_\ell)\right)\,ds_1 \cdots ds_k \\
&=
\int_{\R} \left(p\bigl(y \mid x,\Vert g_{\psi}(\mathrm{z})\Vert s_1\bigr)\,\phi_1(s_1)
\int_{\R^{k-1}}\prod_{\ell=2}^k \phi_1(s_\ell)  \,ds_2 \cdots ds_k  \right) \,  ds_1 \\
&=
\int_{\R} p\bigl(y \mid x,\Vert g_{\psi}(\mathrm{z})\Vert s_1\bigr)\,\phi_1(s_1)\,ds_1.
\end{align}
Setting $\xi \coloneqq s_1$ completes the proof.
\end{proof}

\subsection{Proof of Proposition \ref{prop:binary_new_groups}} \label{sec:proof_binary_new_groups}

\begin{proof}
For any $a,b\in\R$ define
\begin{equation}
h(a,b)
\coloneqq
\int_{-\infty}^{\infty}\eta^{-1}(a\xi+b)\,\phi_1(\xi)\,d\xi
=
\E\bigl[\eta^{-1}(a\xi+b)\bigr],
\qquad \xi\sim\mathcal{N}(0,1).
\end{equation}
Let
\begin{equation}
p_1 \coloneqq p(y_{ij^*}=1\mid x_{ij^*},z_{j^*})
=
h(\Vert g_\psi(z_{j^*})\Vert,f_\theta(x_{ij^*})),
\qquad
p_0 \coloneqq 1-p_1.
\end{equation}
Thus,
\begin{equation}
\arg\max_{c\in\{0,1\}}p(y_{ij^*}=c\mid x_{ij^*},z_{j^*})=1
\ \Leftrightarrow\
h(\Vert g_\psi(z_{j^*})\Vert,f_\theta(x_{ij^*}))-\frac{1}{2}>0.
\end{equation}

For any fixed $a,b\in\R$, using $1-\eta^{-1}(u)=\eta^{-1}(-u)$ and symmetry of $\phi_1$,
\begin{align}
h(a,b)
&=
\int_{-\infty}^{\infty}\eta^{-1}(a\xi+b)\,\phi_1(\xi)\,d\xi \\
&=
\int_{-\infty}^{\infty}\left(1-\eta^{-1}(-a\xi-b)\right)\phi_1(\xi)\,d\xi
=
1-h(a,-b).
\end{align}
In particular, $h(a,0)=1-h(a,0)$ implies
\begin{equation}\label{eq:half}
h(a,0)=\frac{1}{2}.
\end{equation}

Additionally, since $\eta^{-1}$ is strictly increasing, $h(a,b)$ is strictly increasing in $b$ for each fixed $a$.
Combined with \eqref{eq:half}, this yields $h(a,b)>\frac{1}{2}\Leftrightarrow b>0$.
Therefore,
\begin{equation}
\text{sign}\left(h(\Vert g_\psi(z_{j^*})\Vert,f_\theta(x_{ij^*}))-\frac{1}{2}\right)
=
\text{sign}\left(f_\theta(x_{ij^*})\right),
\end{equation}
and the classifier depends only on $f_\theta(x_{ij^*})$.
\end{proof}

\subsection{Proof of Theorem \ref{thm:bound}} \label{sec:bound_proof}

\begin{proof}
Recall that for fixed $(x,z,\theta,\psi)$ and $y$ we have
\begin{equation}
L(\xi)
=
\sum_{i=1}^{n_j}
\left(
y_i\,\vartheta_i(\xi) - A\bigl(\vartheta_i(\xi)\bigr)
\right)
+ \log \phi_1(\xi),
\end{equation}
where
\begin{equation}
\vartheta_i(\xi) \coloneqq \vartheta(x_i,\mathrm{z},\xi;\theta,\psi),
\qquad
\phi_1(\xi) = (2\pi)^{-1/2} \exp\bigl(-\xi^{2}/2\bigr),
\end{equation}
and $x_i$ denotes the $i$-th row of $x$.

Since $p(u \mid \vartheta)$ is a one–parameter exponential family with natural statistic $T(u)=u$, the convex conjugate
\begin{equation}
A^{*}(u) \coloneqq \sup_{\vartheta}\bigl(u\vartheta - A(\vartheta)\bigr)
\end{equation}
is finite for all $u$ in the effective domain of $A^*$.
In particular, for all $\xi \in \R$ and all $i=1,\dots,n_j$, 
\begin{equation}
y_i\,\vartheta_i(\xi) - A\bigl(\vartheta_i(\xi)\bigr)
\le
A^{*}(y_i).
\end{equation}
Summing over $i$ yields
\begin{equation}
\sum_{i=1}^{n_j}
\left(
y_i\,\vartheta_i(\xi) - A\bigl(\vartheta_i(\xi)\bigr)
\right)
\le
\sum_{i=1}^{n_j} A^{*}(y_i),
\end{equation}
and therefore
\begin{equation}
L(\xi)
\le
\sum_{i=1}^{n_j} A^{*}(y_i) + \log \phi_1(\xi)
=
\sum_{i=1}^{n_j} A^{*}(y_i) - \frac{1}{2}\log(2\pi) - \frac{1}{2}\xi^{2}.
\end{equation}
Define
\begin{equation}
\widetilde C(y) \coloneqq \exp\left(\sum_{i=1}^{n_j} A^{*}(y_i)\right)(2\pi)^{-1/2}.
\end{equation}
Then
\begin{equation}\label{eq:envelope_half}
\exp\bigl(L(\xi)\bigr)
\le
\widetilde C(y)\exp\Bigl(-\frac{1}{2}\xi^{2}\Bigr)
\quad\text{for all } \xi \in \R,
\end{equation}
and
\begin{equation}
0 < I(y)
=
\int_{\R} \exp\bigl(L(\xi)\bigr) d\xi
\le
\widetilde C(y) \int_{\R} \exp\Bigl(-\frac{1}{2}\xi^{2}\Bigr) d\xi
<
\infty.
\end{equation}

For $M>0$ define
\begin{equation}
R_M(y) \coloneqq \int_{\vert\xi\vert > M} \exp\bigl(L(\xi)\bigr) d\xi,
\end{equation}
so that
\begin{equation}
I(y) = I_M(y) + R_M(y).
\end{equation}
By definition,
\begin{equation}
\Delta_M(y)
=
\log I(y) - \log I_M(y)
=
\log\left(1 + \frac{R_M(y)}{I_M(y)}\right).
\end{equation}
Since $\log(1+u) \le u$ for all $u\ge0$,
\begin{equation}\label{eq:DeltaM_ratio}
0 \le \Delta_M(y)
\le
\frac{R_M(y)}{I_M(y)}.
\end{equation}

We proceed by deriving upper bounds for $R_M(y)$ and $1/I_M(y)$, starting with the former.
From Equation \eqref{eq:envelope_half},
\begin{equation}
R_M(y)
\le
\widetilde C(y)
\int_{\vert\xi\vert > M} \exp\Bigl(-\frac{1}{2}\xi^{2}\Bigr) d\xi
=
2 \widetilde C(y)
\int_{M}^{\infty} \exp\Bigl(-\frac{1}{2} t^{2}\Bigr) dt.
\end{equation}
For $a>0$ the standard bound
\begin{equation}
\int_{a}^{\infty} \exp\Bigl(-\frac{1}{2} t^{2}\Bigr) dt
\le
\frac{1}{a}\exp\Bigl(-\frac{1}{2} a^{2}\Bigr)
\end{equation}
yields, with $a=M$,
\begin{equation}\label{eq:RM_bound}
R_M(y)
\le
2 \widetilde C(y)\frac{1}{M}\exp\Bigl(-\frac{1}{2} M^{2}\Bigr).
\end{equation}

Now, moving to $1/I_M(y)$, since
\begin{equation}
R_M(y)\longrightarrow 0
\quad \text{as } M \to \infty,
\end{equation}
and $I(y) > 0$, there exists $M_0 = M_0(y;x,z,\theta,\psi)$ such that $R_M(y) \le \frac{1}{2} I(y)$ for all $M \ge M_0$.
For such $M$,
\begin{equation}\label{eq:IM_lower}
I_M(y)
= I(y) - R_M(y)
\ge \frac{1}{2} I(y),
\end{equation}
and hence
\begin{equation}\label{eq:IM_inv}
\frac{1}{I_M(y)}
\le
\frac{2}{I(y)}
\quad\text{for all } M \ge M_0.
\end{equation}

Combining the bounds in \eqref{eq:RM_bound} and \eqref{eq:IM_inv} with Equation \eqref{eq:DeltaM_ratio}, for all $M \ge M_0$ we obtain
\begin{equation}
\Delta_M(y)
\le
\frac{R_M(y)}{I_M(y)}
\le
\frac{2}{I(y)}
\cdot
2 \widetilde C(y)\frac{1}{M}\exp\Bigl(-\frac{1}{2} M^{2}\Bigr)
=
C \frac{1}{M}\exp\Bigl(-\frac{1}{2} M^{2}\Bigr),
\end{equation}
where
\begin{equation}
C \coloneqq C(y;x,z,\theta,\psi)
=
\frac{4 \widetilde C(y)}{I(y)}.
\end{equation}
Taking logarithms and dividing by $M^{2}$ yields
\begin{equation}
\frac{1}{M^{2}} \log \Delta_M(y)
\le
-\frac{1}{2}
+
\frac{1}{M^{2}} \log C
-
\frac{1}{M^{2}} \log M,
\end{equation}
and therefore,
\begin{equation}
\limsup_{M\to\infty}
\frac{1}{M^{2}} \log \Delta_M(y)
\le
-\frac{1}{2}.
\end{equation}
\end{proof}

\section{Multivariate random-effects design} \label{sup:multi}

In the main text we focus on the case where the random-effects design depends only on the group ($z_{ij}=\mathrm{z}$ for all $i$ in the $j$-th group), for which the group likelihood depends on $\gamma\in\mathbb{R}^k$ only through a single scalar projection. In this case Lemma~\ref{lemma:int} reduces the marginalization integral to one dimension.
Here we consider the general case where $z_{ij}$ may vary across subjects within a group.
For a fixed group $(x,y,z)$ and parameters $(\theta,\psi)$,
the random-effects network returns a matrix
\begin{equation}
g_\psi(z)\in\mathbb{R}^{n\times k},
\end{equation}
so that the random-effects contribution is the vector
\begin{equation}
t \coloneqq g_\psi(z)\gamma \in \mathbb{R}^n,
\qquad \gamma\sim\mathcal{N}(0,I_k).
\end{equation}

Since $\gamma$ is Gaussian and $t=g_\psi(z)\gamma$ is a linear map of $\gamma$, the random vector $t$ is (possibly degenerate) Gaussian with
\begin{equation}
t \sim \mathcal{N}\bigl(0,\Sigma\bigr),
\qquad
\Sigma \coloneqq g_\psi(z)\,g_\psi(z)^{\intercal}\in\mathbb{R}^{n\times n}.
\end{equation}
Therefore, the group-level marginalization over $\gamma$ can equivalently be expressed as an expectation with respect to this induced Gaussian distribution of $t$.

Conditionally on $(x,z,\gamma)$, the entries of $y$ are independent and depend on $\gamma$ only through $t=g_\psi(z)\gamma$.
Using the shorthand $p(y_i\mid x_i,t_i;\theta,\psi)$ defined in Equation~\eqref{eq:shorthand_pyxt}, the group marginal likelihood can be written as
\begin{equation}
p(y \mid x,z;\theta,\psi)
=
\int
\left(
\prod_{i=1}^{n} p(y_i\mid x_i,t_i;\theta,\psi)
\right)
\, d\, \mathcal{N}(0,\Sigma)\, (t).
\end{equation}
If $\Sigma$ has rank $r$, we may reduce the integral to $\mathbb{R}^r$ by writing $\Sigma = U\Lambda U^{\intercal}$ with $U\in\mathbb{R}^{n\times r}$ having orthonormal columns and $\Lambda\in\mathbb{R}^{r\times r}$ diagonal with positive entries, and setting $t = U\Lambda^{1/2}s$ with $s\sim\mathcal{N}(0,I_r)$.
In this representation, the group marginal log-likelihood takes the form
\begin{equation}
\log p(y \mid x,z;\theta,\psi)
=
\sum_{i=1}^{n}\log h(y_i)
+
\log \int_{\mathbb{R}^{r}} \exp\bigl(L_r(s; y,x,z,\theta,\psi)\bigr)\,ds,
\end{equation}
for a corresponding multivariate log-integrand $L_r$ that includes the Gaussian density term of $s$.

Analogously to the one-dimensional truncation in Equation~\eqref{eq:trunc}, we define a truncated multivariate log-integral over the Euclidean ball:
\begin{equation}
H_r(M; y,x,z,\theta,\psi)
\coloneqq
\log \int_{\Vert s\Vert \le M} \exp\bigl(L_r(s; y,x,z,\theta,\psi)\bigr)\,ds.
\end{equation}
This quantity reduces to $H(\xi_1; y,x,z,\theta,\psi)$ from the main text when $r=1$.

\paragraph{\emph{Optimization.}}
In the general case of $r>1$, the one-dimensional ODE representation is no longer available.
Instead,  $H_r(M; y,x,z,\theta,\psi)$ can be approximated by a differentiable quadrature rule.
Let $\{(s_q,w_q)\}_{q=1}^{Q}$ be quadrature nodes $s_q\in\mathbb{R}^r$ and weights $w_q>0$ supported on $\{s:\Vert s\Vert\le M\}$, and define
\begin{equation}
\tilde H_r(M; y,x,z,\theta,\psi)
\coloneqq
\log\left(
\sum_{q=1}^{Q} w_q \exp\bigl(L_r(s_q; y,x,z,\theta,\psi)\bigr)
\right).
\end{equation}
This is a log-sum-exp expression, and hence is differentiable with respect to $(\theta,\psi)$.
Gradients can be obtained by automatic differentiation through this quadrature approximation, enabling end-to-end training under general random-effects designs.
In this case, Algorithm~\ref{alg:truncated_ode_sgd} is modified by replacing the ODE solver for $H(\xi_1)$ with quadrature evaluation of $\tilde H_r(M)$ for each sampled group.

\paragraph{\emph{Truncation bounds.}}
The truncation analysis in Section~\ref{sec:theory} extends to the multivariate integral above.
The key step in Theorem~\ref{thm:bound} is an envelope bound of the form
\begin{equation}
\exp(L(\xi)) \le \widetilde C(y)\exp\left(-\frac{1}{2}\xi^2\right),
\end{equation}
which follows from bounding the exponential-family term by a constant (via the convex conjugate $A^*$) and retaining the Gaussian prior term.
In the multivariate representation, the same argument yields an envelope
\begin{equation}
\exp\bigl(L_r(s)\bigr) \le \widetilde C(y)\exp\left(-\frac{1}{2}\Vert s\Vert^2\right),
\qquad s\in\mathbb{R}^r,
\end{equation}
so the truncation remainder over $\{\Vert s\Vert>M\}$ is bounded by a Gaussian tail integral in dimension $r$.
Consequently, the difference between the full and truncated log-integrals,
\begin{equation}
\Delta_{r,M}(y;x,z,\theta,\psi)
\coloneqq
\log\int_{\mathbb{R}^{r}} \exp(L_r(s))\,ds
-
\log\int_{\Vert s\Vert\le M} \exp(L_r(s))\,ds,
\end{equation}
decays at a Gaussian-tail rate in $M^2$ (up to a dimension-dependent polynomial prefactor), matching the worst-case exponent $1/2$ from the one-dimensional bound.

\section{Linear simulations} \label{sup:simualations}

Similarly to the simulation results in Section \ref{sec:synthetic}, here we report results for analogous experiments for data generated from a linear mixed-effects model.

The grouped structure, covariate distribution, and outcome families are identical to the nonlinear setting, but here we consider a linear fixed-effect function
\begin{equation}
f^{*}\left(x^{(1)},x^{(2)}\right)
=
\beta_0
+
\beta_1 x^{(1)}
+
\beta_2 x^{(2)},
\qquad
(\beta_0,\beta_1,\beta_2) = (0,0.1,0.1).
\end{equation}
This yields
\begin{equation}
\eta_{ij}
=
f^{*}\left(x^{(1)}_{ij}, x^{(2)}_{ij}\right)
+
\gamma_j
\end{equation}
for the \emph{random intercept} setting, and
\begin{equation}
\eta_{ij}
=
f^{*}\left(x^{(1)}_{ij}, x^{(2)}_{ij}\right)
+
\gamma_j x^{(1)}_{ij}
\end{equation}

In these linear experiments we reduced our neural-networks to a linear architecture, with both the fixed and random components represented by single linear layers.
As expected, Figure \ref{fig:simulations_linear} shows that under correct linear specification, GLMMs outperform our approach in approximately $70\%$ of settings ($25/36$), with most differences smaller in magnitude compared to the gains observed in the nonlinear setting.

\begin{figure} [ht]
    \centering
    \includegraphics[width=\linewidth]{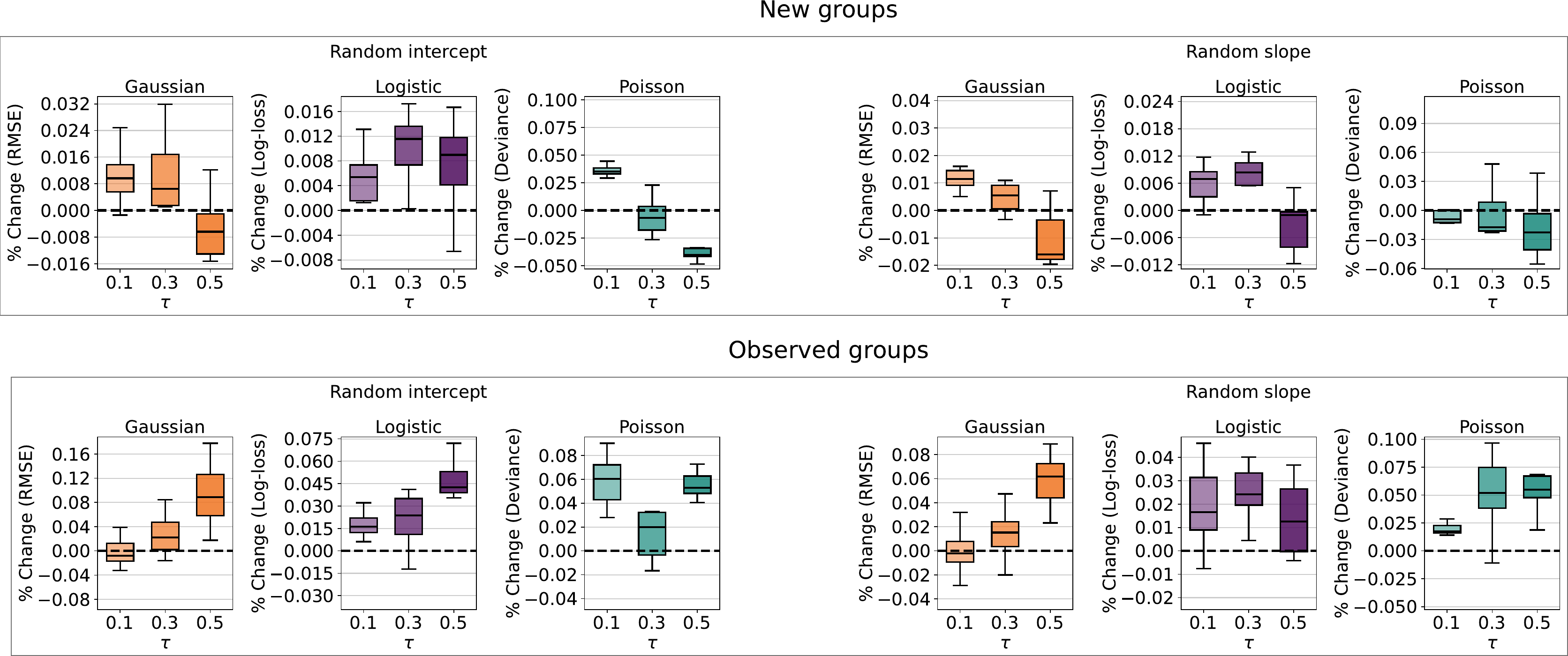}
    \caption{Simulation results for linear links.}
    \label{fig:simulations_linear}
\end{figure}

\section{Experimental comparison to BBVI} \label{sec:bbvi}

In this paper we developed a tailored approximate marginal-likelihood objective to fit the model in Equations \eqref{eq:our_model} and \eqref{eq:prior}. Alternatively, the same model could be learned using generic procedures such as black-box variational inference (BBVI) \citep{ranganath2014black}. In Figures \ref{fig:airbnb_bbvi} and \ref{fig:rxrx_bbvi} we show empirically that NGMM outperforms BBVI on the Airbnb and RxRx1 datasets respectively.

\begin{figure} [ht]
    \centering
    \includegraphics[width=0.8\linewidth]{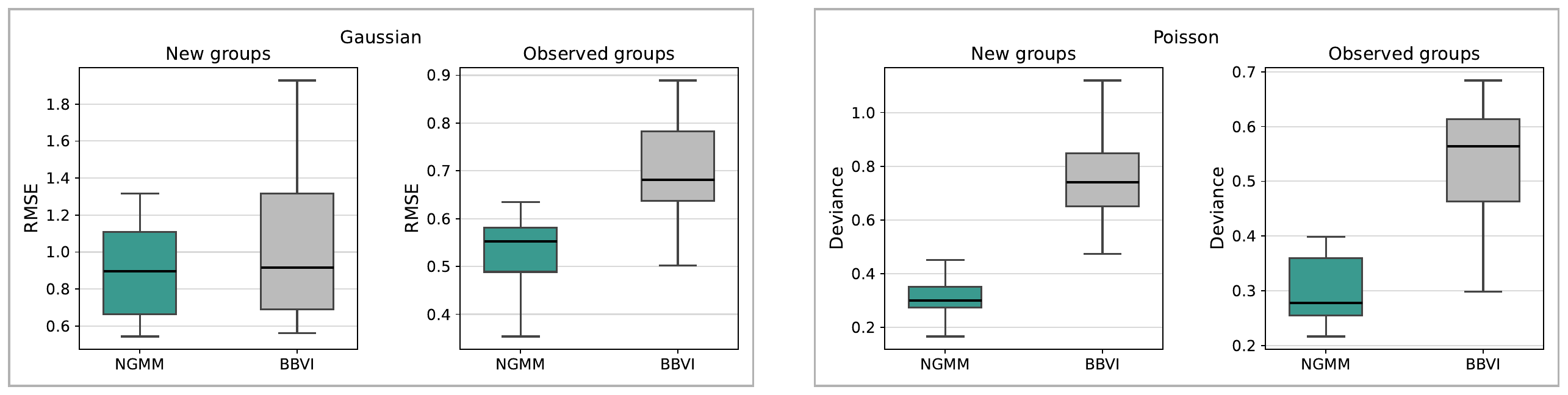}
    \caption{Airbnb experimental comparison of NGMM (this paper) to BBVI \citep{ranganath2014black}.}
    \label{fig:airbnb_bbvi}
\end{figure}

\begin{figure} [ht]
    \centering
    \includegraphics[width=0.8\linewidth]{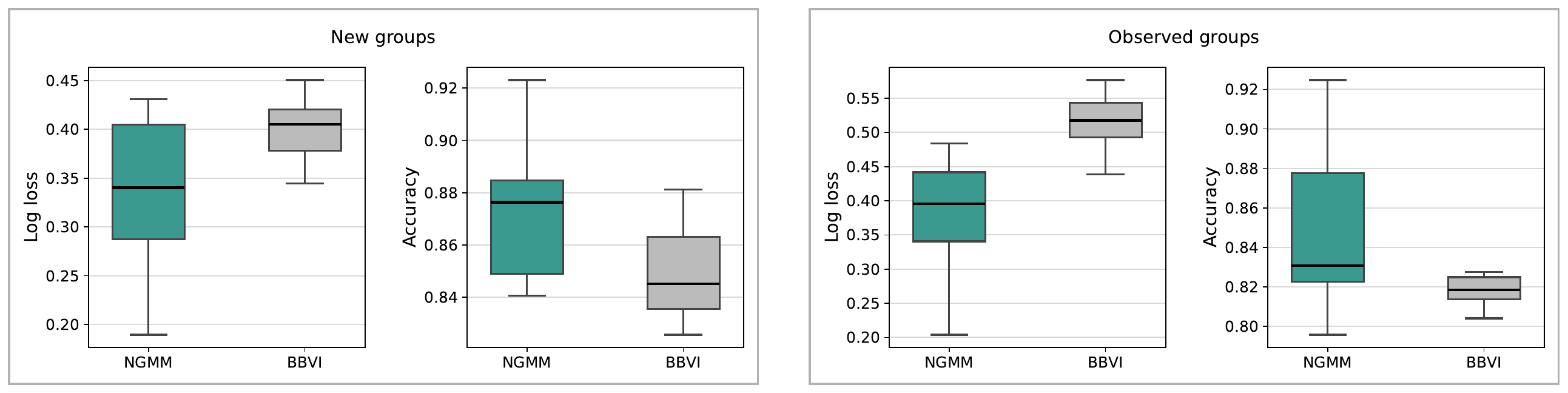}
    \caption{RxRx1 experimental comparison of NGMM (this paper) to BBVI \citep{ranganath2014black}.}
    \label{fig:rxrx_bbvi}
\end{figure}

\section{Additional results for PISA analysis} \label{sup:pisa_results}

We provide the experimental results for the science and reading domains in Figures \ref{fig:pisa_science} and \ref{fig:pisa_reading}, respectively. As in the math results (see Figure \ref{fig:pisa_math}), 
in all six countries, neural models outperform their linear counterparts. For previously observed schools, incorporating random effects substantially improves performance in both linear and neural models. For new schools, the incremental benefit of adding random effects to the neural model is modest.

\begin{figure} [ht]
    \centering
    \includegraphics[width=\linewidth]{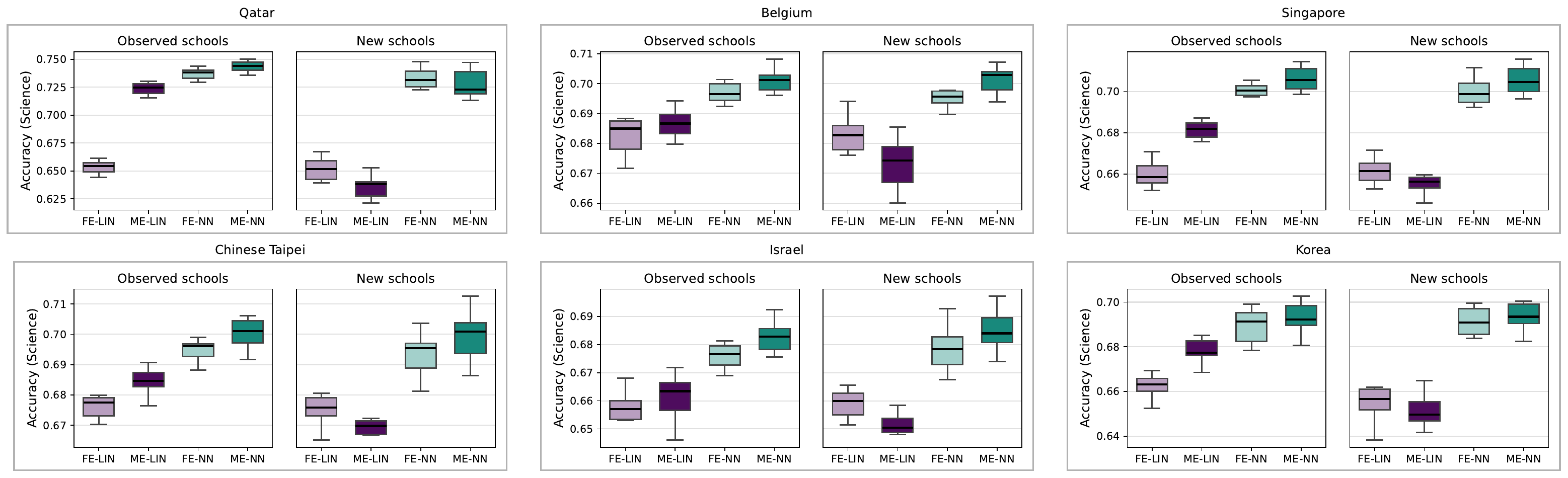}
    \caption{PISA student proficiency in science: experimental results.}
    \label{fig:pisa_science}
\end{figure}

\begin{figure} [ht]
    \centering
    \includegraphics[width=\linewidth]{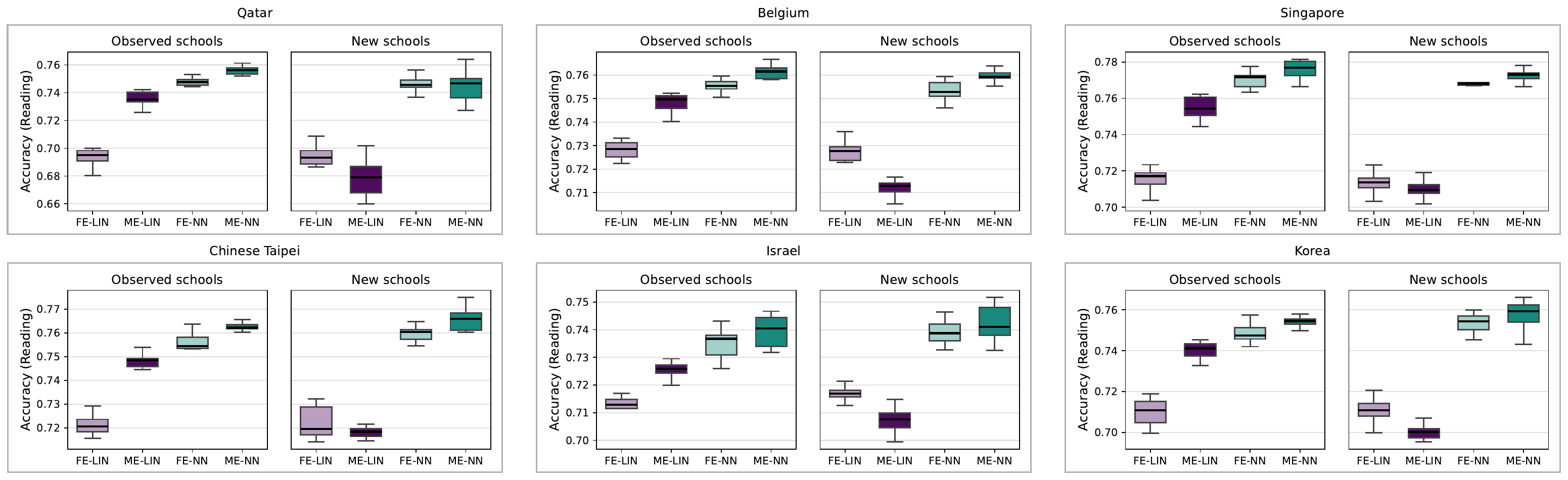}
    \caption{PISA student proficiency in reading: experimental results.}
    \label{fig:pisa_reading}
\end{figure}

\section{Additional experimental details} 

\subsection{Synthetic studies} \label{sup:synthetic}
For the Gaussian model we set the observation noise standard deviation to $\sigma_\varepsilon = 1.0$. 
In each repetition of the experiment, we split the corresponding data into training, test-existing, and test-new subsets. We assign $\lfloor 0.7\,m \rfloor$ groups to the training set and the remaining groups to the held-out (new) set. Within each training group we randomly assign $\lfloor 0.7\,n_j \rfloor$ observations to the training subset and keep the remainder as the existing-groups test subset.

For the Gaussian random-effects model we use the \texttt{MixedLM} implementation from \texttt{statsmodels}. For the logistic and Poisson cases we fit generalized linear mixed models using the \texttt{glmer} function from the \texttt{lme4} package in R.

\subsection{Airbnb Experiments} \label{sup:airbnb}

\paragraph{\emph{Data pre-processing.}}
We retain only numeric covariates with at most 5 covariates 
5\% missing values. Remaining missing entries are imputed with the training-set median of the corresponding covariate. 
All covariates were then standardized to zero mean and unit variance using statistics computed on the training data, and the same transformation was applied to the test data. For the Gaussian task the response was set to the logarithm of the listing price, and was standardized using the training-set mean and standard deviation. For the Poisson task the response was set to bedroom count. 

\paragraph{\emph{Baseline neural architectures.}}
 In the OHE baseline the host indicator was represented as a one-hot vector and concatenated to the covariates before the MLP. In the Embed baseline each host is assigned a learned embedding of dimension 16, which is concatenated to the covariates and fed through the same MLP. 

\paragraph{\emph{Training and hyperparameters.}}
Hyper parameters for all methods were chosen with a grid search on one repetition of the experiment that was excluded for the analysis.
For NGMM on both Gaussian and Poisson responses we used Adam optimizer with learning rate $10^{-3}$, mini-batches of size 128, and 100 epochs. The random intercept was integrated over a fixed symmetric interval in the latent space using the same ODE solver settings as in the synthetic experiments. For BBVI we place independent Gaussian priors on all network parameters and random effects and use a fully factorized Gaussian variational family. In the Gaussian task we ran 5\,000 gradient steps of Adam with learning rate $10^{-3}$ and 5 Monte Carlo samples per step. In the Poisson task we used 10\,000 steps with the same settings. The OHE and Embed baselines were also optimized with Adam and learning rate $10^{-3}$, batch size 128, and 500 epochs.

\subsection{RxRx1 Experiment} \label{sup:rxrx}

Here we detail the setup for the RxRx1 microscopy experiment in Section~\ref{sec:rxrx}.

\paragraph{\emph{Binary task construction.}}
We fixed the untreated control label to the value used for EMPTY wells in the metadata and treated it as the negative class. In each repetition we selected one active siRNA in the training split and defined the positive class as this siRNA. We then restricted the dataset to images whose labels are either the chosen treated label or the control label and mapped the response to a binary variable with value 1 for treated and 0 for control.

\paragraph{\emph{Image pre-processing.}}
We worked directly with the microscopy images. Each image was resized to $256 \times 256$ pixels, centered, and cropped to this size. 

\paragraph{\emph{Neural architectures and baselines.}}
 For models with random effects, the fixed-effects logit was augmented by a scalar random intercept per group, so that the total log-odds for observation $(i,j)$ is $\eta_{ij} = \eta_x(x_{ij}) + \gamma_j$. 

In the OHE baseline groups were represented by a one-hot vector, concatenated to the CNN representation. The embedding baseline replaced the one-hot representation with a learned group embedding of dimension 64, which was concatenated to the CNN features and passed through the same classifier. 

\paragraph{\emph{Training and hyperparameters.}}
Hyper parameters for all methods were chosen with a grid search on one repetition of the experiment that was excluded for the analysis.
For NGMM we used Adam optimizer with learning rate $10^{-3}$. We ran the experiment with batch size 64 for 350 epochs in both evaluation regimes. The random-intercept integral was approximated over a fixed symmetric interval in the latent space using 32 ODE steps. For BBVI we use the same prior structure and ran 3\,000 gradient steps of Adam with learning rate $10^{-3}$, batch size 64, and 5 Monte Carlo samples per gradient estimate. The OHE and embedding baselines were trained with Adam optimizer, learning rate $10^{-3}$, batch size 64, and 300 epochs.

\subsection{PISA analysis} \label{sup:pisa}

As in the original PISA analysis, we analyze each country separately. 

\paragraph{\emph{Data pre-processing.}}
For a given country, we restrict to students and schools appearing in the cognitive assessment roster, and drop schools with fewer than $5$ participating students.
Background covariates $x_{ij}$ and $z_j$ are constructed from the student, parent, and school questionnaires while ignoring sampling weights and plausible values.  
We remove  variables with more than $98\%$ missing values or with more than $500$ distinct categories. For linear models we keep PCA components explaining $95\%$ of the variance, as in the original PISA analysis.
For neural models, we standardize continuous variables and augment categorical values with an explicit missing-value category.

We split the data by sampling schools uniformly at random to train ($80\%$), validation ($10\%$), and test ($10\%$). Within the training schools, we further hold out $20\%$ of the students in each training school to form a “same-school” test set, and use the remaining students for fitting.

\paragraph{\emph{Training and hyperparameters.}}

All four considered models (fixed-effects linear model,  random-effects linear model, fixed-effects neural model, neural mixed-effects model) are fit by maximizing the marginal log-likelihood. For the neural models,  the inner integral over $\varphi_{ij}$ in Equation \eqref{eq:marginal_likelihood} and the 
outer integral over school effects $\gamma_j$ are approximated with Gauss–Hermite (we use $10$ nodes for each integral). 

The neural models share the same architecture: categorical covariates are embedded and concatenated with standardized continuous covariates, then passed through a transformer encoder with model dimension $128$, $4$ self-attention layers, $8$ heads.
In the fixed-effects neural model this architecture maps $x_{ij}$ to a three-dimensional proficiency vector; in the neural mixed-effects model we use one such network for $x_{ij}$ and a second, smaller network for $z_j$ that produces a three-dimensional scale $g_{\psi}(z_j)$ for the school-level random effect in Equation \eqref{eq:pisa_neural_phi}.

All models are optimized with Adam optimizer with learning rate $10^{-4}$ using mini-batches of schools (up to $32$ schools and $4,000$ students per batch) for $1,000$ training iterations. 

\end{appendices}
\end{document}